\documentclass[smallextended]{svjour3}

\usepackage{graphicx}
\usepackage{epstopdf}
\usepackage{algorithm}
\usepackage{algorithmicx}
\usepackage{algpseudocode}
\usepackage{cite}
\usepackage{color}
\usepackage{amsmath,amssymb,amsfonts}
\usepackage{bm}
\usepackage{graphicx}
\usepackage{epstopdf}
\usepackage{subcaption}
\usepackage{multirow}
\usepackage{mathtools}
\usepackage{ctable}
\usepackage{bbm}
\usepackage{verbatim}

\newtheorem{rem} {Remark}

\newcommand{\lb}{\left(}
\newcommand{\rb}{\right)}

\newcommand{\bA}{\mathbf{A}}
\newcommand{\bV}{\mathbf{V}}
\newcommand{\bW}{\mathbf{W}}
\newcommand{\bWN}{\mathbf{W}_{\mathcal{N}}}

\newcommand{\bL}{\mathbf{L}}
\newcommand{\btL}{\widetilde{\mathbf{L}}}
\newcommand{\bLN}{\mathbf{L}_{\mathcal{N}}}
\newcommand{\btLN}{\widetilde{\mathbf{L}}_{\mathcal{N}}}

\newcommand{\bLambda}{\mathbf{\Lambda}}
\newcommand{\bM}{\mathbf{M}}
\newcommand{\bT}{\mathbf{T}}
\newcommand{\bQ}{\mathbf{Q}}
\newcommand{\bU}{\mathbf{U}}
\newcommand{\bR}{\mathbf{R}}
\newcommand{\bS}{\mathbf{S}}

\newcommand{\bone}{\mathbf{1}}
\newcommand{\bzero}{\mathbf{0}}

\newcommand{\bx}{\mathbf{x}}
\newcommand{\by}{\mathbf{y}}

\newcommand{\bu}{\mathbf{u}}

\newcommand{\hatm}{\widehat{m}}
\newcommand{\hn}{\widehat{n}}

\newcommand{\hG}{\widehat{G}}

\newcommand{\bX}{\mathbf{X}}

\newcommand{\bI}{\mathbf{I}}
\newcommand{\bO}{\mathbf{O}}

\newcommand{\bv}{\mathbf{v}}

\newcommand{\trace}{\textnormal{trace}}
\newcommand{\diag}{\textnormal{diag}}

\title{ Incremental Eigenpair Computation for Graph Laplacian Matrices: Theory and Applications\footnote{This manuscript is an extended version of a paper that was presented at
ACM KDD Workshop on Mining and Learning with Graphs (MLG 2016)~\cite{chen2016incremental}.
}}

\author{Pin-Yu Chen \and
	Baichuan Zhang \and Mohammad Al Hasan 
}

\institute{
Pin-Yu Chen \at
		AI Foundations - Learning Group, IBM Thomas J. Watson Research Center \\
		\email{\{pin-yu.chen@ibm.com\}} \\
Baichuan Zhang \at 
	   Purdue University, West Lafayette \\
       \email{\{zhan1910@purdue.edu\}} \\
Mohammad Al Hasan \at
              Indiana University Purdue University Indianapolis \\
              \email{\{alhasan@cs.iupui.edu\}}   
}

\begin{document}

\maketitle

\begin{abstract}

The smallest eigenvalues and the associated eigenvectors (i.e., eigenpairs) of a graph Laplacian 
matrix have been widely used in spectral clustering and community detection.  
However, in real-life applications the number of clusters or communities (say, $K$) 
is generally unknown a-priori. Consequently, the majority of the existing methods either choose $K$ heuristically 
or they repeat the clustering method with different choices of $K$ and accept the best clustering result.  
The first option, more often, yields suboptimal result, while the second option is computationally expensive.  
In this work, we propose an incremental method for constructing the eigenspectrum of the graph Laplacian matrix. 
This method leverages the eigenstructure of graph Laplacian matrix to obtain the $K$-th smallest eigenpair of the Laplacian matrix 
given a collection of all previously computed $K-1$ smallest eigenpairs. Our proposed method adapts the Laplacian matrix such that 
the batch eigenvalue decomposition problem transforms into an efficient sequential leading eigenpair computation problem. 
As a practical application, we consider user-guided spectral clustering. 
Specifically, we demonstrate that users can utilize the proposed incremental method 
for effective eigenpair computation and for determining the desired number of clusters based on multiple clustering metrics.

\end{abstract}



\section{Introduction}

Over the past two decades, the graph Laplacian matrix and its variants have
been widely adopted for solving various research tasks, including graph
partitioning \cite{pothen1990partitioning}, data clustering \cite{Luxburg07,wu2016revisiting,chen2016efficient},
community detection \cite{White05,CPY14spectral,chen2017revisiting}, consensus in networks
\cite{Olfati-Saber07,wu2009empirical}, accelerated distributed optimization \cite{liu2017accelerated}, 
dimensionality reduction \cite{belkin2003laplacian,wu2016estimating},
entity disambiguation \cite{ZhangSaha14,Zhang.Dundar.ea:16,Zhang.Hasan.3,tkde_online,Zhang.Noman.ea:17,Zhang.Hasan.ea:15}, link prediction \cite{Hasan06linkprediction,Hasan2011, Zhang.Hasan.ea:16,icde-sutanay}, graph signal processing
\cite{Shuman13,chen2017bias}, centrality measures for graph connectivity \cite{CPY14deep}, multi-layer network analysis \cite{chen_tsipn,weiyi}, 
interconnected physical systems \cite{Radicchi13}, network vulnerability
assessment \cite{CPY14ComMag}, image segmentation \cite{Shi00,icmla-dundar}, gene expression~\cite{zifan1,zifan2,liu2017genome}, among others.
The fundamental task is to represent the data of interest as a graph for
analysis, where a node represents an entity (e.g., a pixel in an image or a user in an online
social network) and an edge represents similarity between two multivariate data
samples or actual relation  (e.g., friendship) between nodes \cite{Luxburg07}.
More often the $K$ eigenvectors associated with the $K$ smallest eigenvalues of
the graph Laplacian matrix are used to cluster the entities into $K$ clusters
of high similarity. For brevity, throughout this paper we will call these
eigenvectors as the $K$ smallest eigenvectors. 

The success of graph Laplacian matrix based methods for graph partitioning and
spectral clustering can be explained by the fact that acquiring  $K$ smallest
eigenvectors is equivalent to solving a relaxed graph cut minimization problem,
which partitions a graph into $K$ clusters by minimizing various objective
functions including min cut, ratio cut or normalized cut \cite{Luxburg07}.
Generally, in clustering the value $K$ is selected to be much smaller than $n$ (the
number of data points), making full eigen decomposition (such as QR decomposition) unnecessary. An efficient alternative
is to use methods that are based on power iteration, such as Arnoldi method or
Lanczos method, which computes the leading eigenpairs through repeated matrix
vector multiplication \cite{wu2015preconditioned,wu2017primme_svds}. 
ARPACK \cite{lehoucq1998arpack}
library is a popular parallel implementation of different variants of Arnoldi
and Lanczos methods, which is used by many commercial software including Matlab.

However, in most situations the best value of $K$ is unknown and a heuristic is
used by clustering algorithms to determine the number of clusters, e.g., fixing
a maximum number of clusters $K_{\max}$ and detecting a large gap in the values
of the $K_{\max}$ largest sorted eigenvalues or normalized cut score
\cite{Polito01grouping,ng2002spectral}.  Alternatively, this value of $K$ can
be determined based on domain knowledge \cite{basu2004active}. For example, a
user may require that the largest cluster size be no more than 10\% of the
total number of nodes or that the total inter-cluster edge weight be no greater
than a certain amount.  In these cases, the desired choice of $K$ cannot be
determined \textit{a priori}.  Over-estimation of the upper bound $K_{\max}$  on the
number  of clusters is expensive as the cost of finding $K$ eigenpairs using
the power iteration method grows rapidly with $K$.  On the other hand, choosing
an insufficiently large value for $K_{\max}$ runs the risk of severe bias.
Setting $K_{\max}$ to the number of data points $n$ is generally
computationally infeasible, even for a moderate-sized graph.  Therefore, an
incremental eigenpair computation method that effectively computes the $K$-th
smallest eigenpair of graph Laplacian matrix by utilizing the previously
computed $K-1$ smallest eigenpairs is needed.  Such an iterative method
obviates the need to set an upper bound $K_{\max}$ on $K$, and its efficiency
can be explained by the adaptivity to increments in $K$.

By exploiting the special matrix properties and graph characteristics of a graph Laplacian matrix, we propose an efficient method for computing the $(K+1)$-th eigenpair given all of the $K$ smallest eigenpairs, which we call the Incremental method of Increasing Orders (Incremental-IO).
For each increment, given the previously computed
smallest eigenpairs, we show that computing the next smallest eigenpair is
equivalent to computing a leading eigenpair of a particular matrix, which
transforms potentially tedious numerical computation (such as the
iterative tridiagonalization and eigen-decomposition steps in the Lanczos algorithm \cite{lanczos1950iteration}) to
simple matrix power iterations of known computational efficiency
\cite{kuczynski1992estimating}. Specifically, we show that Incremental-IO can be implemented via power iterations, and analyze its computational complexity and data storage requirements.
We then compare the performance of Incremental-IO with a batch computation method which computes all of the $K$ smallest eigenpairs in a single batch, and an incremental method adapted from the Lanczos algorithm, which we call the Lanczos method of Increasing Orders (Lanczos-IO). 
For a given number of eigenpairs $K$ iterative
matrix-vector multiplication of Lanczos procedure yields a set of Lanczos
vectors ($\mathbf{Q}_\ell$), and a tridiagonal matrix ($\mathbf{T}_\ell$), followed
by a full eigen-decomposition of $\mathbf{T}_\ell$ ($\ell$ is a value much smaller
than the matrix size). 
Lanczos-IO saves the Lanczos vectors that were obtained while
$K$ eigenpairs were computed and used those to generate new Lanczos vectors for
computing the $(K+1)$-th eigenpair.

Comparing to the batch method, our experimental results show that for a given
order $K$, Incremental-IO provides a significant reduction in
computation time. Also, as $K$ increases, the gap
between Incremental-IO and the batch approach widens, providing an
order of magnitude speed-up. Experiments on
real-life datasets show that the performance of Lanczos-IO is overly sensitive
to the selection of augmented Lanczos vectors, a parameter that cannot be
optimized \textit{a priori}---for some of our experimental datasets, Lanczos-IO
performs even worse than the batch method (see Sec.~\ref{sec:ex}).
Moreover, Lanczos-IO consumes significant amount of memory as it has to save
the Lanczos vectors ($\mathbf{Q}_\ell$) for making the incremental approach
realizable. In summary, Lanczos-IO, although an incremental eigenpair computation
algorithm, falls short in terms of robustness.

To illustrate the real-life utility of incremental eigenpair computation methods, we
design a user-guided spectral clustering algorithm which uses Incremental-IO.  The algorithm provides clustering solution for a sequence of $K$
values efficiently, and thus enable a user to compare these clustering
solutions for facilitating the selection of the most appropriate clustering.

\pagebreak 
The contributions of this paper are summarized as follows:
\begin{enumerate}
	\item We propose an incremental eigenpair computation method (Incremental-IO) for both
unnormalized and normalized graph Laplacian matrices, by transforming the
original eigenvalue decomposition problem into an efficient sequential leading
eigenpair computation problem. Specifically, Incremental-IO can be implemented via power iterations, which possess efficient computational complexity and data storage. Simulation results show that Incremental-IO
generates the desired eigenpair accurately and has superior performance over the
batch computation method in terms of computation time.  

\item We show that Incremental-IO is robust in comparison to Lanczos-IO,
  which is an incremental eigenpair method that we design by adapting the
  Lanczos method.
\item We use several real-life datasets to demonstrate the utility of Incremental-IO. Specifically, we show that
	Incremental-IO is suitable for user-guided spectral clustering which provides
	a sequence of clustering results for unit increment of the number $K$ of clusters, and updates the associated
	cluster evaluation metrics for helping a user in decision making. \end{enumerate}

\section{Related Work}
\label{sec_related}
%

\subsection{Incremental eigenvalue decomposition}
The proposed method (Incremental-IO) aims to incrementally compute the smallest eigenpair of a
given graph Laplacian matrix. 
 There are
several works that are named as incremental eigenvalue decomposition methods
\cite{ning2007incremental,jia2009incremental,ning2010incremental,dhanjal2014efficient,ranjan2014incremental}.
However, these works focus on updating the eigenstructure of graph Laplacian matrix
of dynamic graphs when nodes (data samples) or edges are inserted or
deleted from the graph, which are fundamentally different from incremental computation of eigenpairs of increasing orders.
Consequently, albeit similarity in research topics, they are two distinct sets of problems and cannot be directly compared.

\subsection{Cluster Count Selection for Spectral Clustering}
Many spectral clustering algorithms utilize the eigenstructure of graph
Laplacian matrix for selecting number of clusters. In \cite{Polito01grouping},
a value $K$ that maximizes the gap
between the $K$-th and the $(K+1)$-th smallest eigenvalue is selected as the
number of clusters. In \cite{ng2002spectral}, a value $K$ that minimizes the
sum of cluster-wise Euclidean distance between each data point and the centroid
obtained from K-means clustering on $K$ smallest eigenvectors is selected as
the number of clusters. In \cite{zelnik2004self}, the smallest eigenvectors of
normalized graph Laplacian matrix are rotated to find the best alignment that
reflects the true clusters.  A model based method for determining the number of
clusters is proposed in \cite{Poon2012}.  
In \cite{chen2016phase}, a model order selection criterion for identifying the number of clusters is proposed by estimating the interconnectivity of the graph using
the eigenpairs of the graph Laplacian matrix.  In \cite{van2001graph}, the
clusters are identified via random walk on graphs. In \cite{blondel2008fast},
an iterative greedy modularity maximization approach is proposed for
clustering. In \cite{Krzakala2013,Saade2015spectral}, the eigenpairs of the
nonbacktracking matrix are used to identify clusters.
Note that aforementioned methods use
only one single clustering metric to determine the number of clusters and often
implicitly assume an upper bound on $K$ (namely $K_{\max}$).
As demonstrated in Sec. \ref{sec:ex}, the proposed incremental eigenpair computation method (Incremental-IO) can be used to efficiently provide a sequence of clustering results for unit increment of the number $K$ of clusters and updates the associated (potentially multiple) cluster evaluation metrics.

\section{Incremental Eigenpair Computation for Graph Laplacian Matrices}
\label{sec_incremental_eigenpair}
\subsection{Background}
Throughout this paper bold uppercase letters (e.g., $\bX$) denote matrices
and $\bX_{ij}$ (or $[\bX]_{ij}$) denotes the entry in $i$-th row and $j$-th
column of $\bX$, bold lowercase letters (e.g., $\bx$ or $\bx_i$) denote
column vectors, $(\cdot)^T$ denotes matrix or vector transpose, italic
letters (e.g., $x$ or $x_i$) denote scalars, and calligraphic uppercase letters
(e.g., $\mathcal{X}$ or $\mathcal{X}_i$) denote sets.  The $n \times 1$ vector
of ones (zeros) is denoted by $\bone_n$ ($\bzero_n$). The matrix $\bI$ denotes
an identity matrix and the matrix $\bO$ denotes the matrix of zeros.

We use two $n \times n$ symmetric matrices, $\bA$ and $\bW$, to denote the
adjacency and weight matrices of an undirected weighted simple graph $G$ with $n$
nodes and $m$ edges. $\bA_{ij}=1$ if there is an edge between nodes $i$ and
$j$, and $\bA_{ij}=0$  otherwise. $\bW$ is a nonnegative symmetric matrix such that
$\bW_{ij} \geq 0$ if $\bA_{ij}=1$ and $\bW_{ij}=0$ if $\bA_{ij}=0$. Let
$s_i=\sum_{j=1}^n \bW_{ij}$ denote the strength of node $i$. Note that when
$\bW=\bA$, the strength of a node is equivalent to its degree. 
$\bS=\diag ( s_1, s_2, \ldots, s_n )$ is a diagonal matrix with nodal strength
on its main diagonal and the off-diagonal entries being zero. 

The (unnormalized) graph Laplacian matrix is defined as 
\begin{align}
\label{eqn_unnormalized_graph_Laplacian}
\bL=\bS-\bW.
\end{align}
One popular variant of the graph Laplacian matrix is the normalized graph Laplacian matrix defined as 
\begin{align}
\label{eqn_normalized_graph_Laplacian}
\bLN=\bS^{-\frac{1}{2}} \bL \bS^{-\frac{1}{2}} = \bI - \bS^{-\frac{1}{2}} \bW \bS^{-\frac{1}{2}},
\end{align}
where  $\bS^{-\frac{1}{2}}=\diag ( \frac{1}{\sqrt{s_1}}, \frac{1}{\sqrt{s_2}},
\ldots, \frac{1}{\sqrt{s_n}}  )$. 
The $i$-th smallest eigenvalue and its associated unit-norm eigenvector of
$\bL$ are denoted by $\lambda_i(\bL)$ and $\bv_i (\bL)$, respectively. That is,
the eigenpair $(\lambda_i,\bv_i)$ of $\bL$ has the relation $\bL
\bv_i=\lambda_i \bv_i$, and $\lambda_1(\bL) \leq \lambda_2(\bL) \leq \ldots
\leq \lambda_n (\bL)$. The eigenvectors have unit Euclidean norm and they are
orthogonal to each other such that $\bv_i^T \bv_j=1$ if $i=j$ and $\bv_i^T
\bv_j=0$ if $i \neq j$. The eigenvalues of $\bL$ are said to be distinct if
$\lambda_1(\bL)< \lambda_2(\bL)<\ldots<\lambda_n(\bL)$.  Similar notations are
used for $\bLN$. 


\begin{table*}[t]
	\centering	
	{ \scriptsize	
	\caption{Utility of the established lemmas, corollaries, and theorems.}
	\label{table_utility_incremental}
	\begin{tabular}{c|c|c}
		\hline
		Graph Type / Graph Laplacian Matrix & Unnormalized                                 & Normalized                                     \\ \hline
		Connected Graphs              & \textbf{Lemma} \ref{lem_connect_GL}, \textbf{Theorem} \ref{thm_connect_GL}   & \textbf{Corollary} \ref{cor_connect_NGL}, \textbf{Corollary} \ref{cor_connect_NGL_incremental} \\ \hline
		Disconnected Graphs           & \textbf{Lemma} \ref{lem_disconnect_GL}, \textbf{Theorem} \ref{thm_disconnect_GL}  & \textbf{Corollary} \ref{cor_disconnect_NGL}, \textbf{Corollary} \ref{cor_disconnect_NGL_incremental}                           \\ \hline
	\end{tabular}
	}	
\end{table*}

\subsection{Theoretical foundations of the proposed method (Incremental-IO)}
\label{subsec_Incremental_pre}

The following lemmas and corollaries provide the cornerstone for establishing
the proposed incremental eigenpair computation method (Incremental-IO). The main idea is that we utilize the
eigenspace structure of graph Laplacian matrix to inflate specific eigenpairs
via a particular perturbation matrix, without affecting other eigenpairs.
Incremental-IO can be viewed as a specialized Hotelling's deflation method \cite{parlett1980symmetric} designed for graph Laplacian matrices by exploiting their spectral properties and associated graph characteristics.
It works for both connected, and disconnected graphs using either
normalized or unnormalized graph Laplacian matrix. For illustration purposes, in Table \ref{table_utility_incremental} we group the
established lemmas, corollaries, and theorems under different graph types and
different graph Laplacian matrices.

\begin{lemma}
	\label{lem_connect_GL}
	Assume that $G$ is a connected graph and $\bL$ is the graph Laplacian with $s_i$ denoting the sum of the entries in the $i$-th row of the weight matrix $\bW$. Let $s=\sum_{i=1}^n s_i$ and define
	$\btL=\bL+\frac{s}{n} \bone_n \bone_n ^T.$ 
	Then the eigenpairs of $\btL$ satisfy $(\lambda_i(\btL),\bv_i(\btL))=(\lambda_{i+1}(\bL),\bv_{i+1}(\bL))$ for $1 \leq i \leq n-1$ and $(\lambda_n(\btL),\bv_n(\btL))=(s,\frac{\bone_n}{\sqrt{n}})$.
\end{lemma}
\begin{proof}
		Since $\bL$ is a positive  semidefinite (PSD) matrix \cite{Chung97SpectralGraph}, $\lambda_i(\bL) \geq 0$ for all $i$. Since $G$ is a connected graph, by (\ref{eqn_unnormalized_graph_Laplacian}) $\bL \bone_n=(\bS - \bW) \bone_n=\bzero_n$. Therefore, by the PSD property we have $(\lambda_1(\bL),\bv_1(\bL))=(0,\frac{\bone_n}{\sqrt{n}})$. Moreover, since $\bL$ is a symmetric real-valued square matrix,
	from (\ref{eqn_unnormalized_graph_Laplacian}) we have 
	\begin{align}
	\label{eqn_trace_connect_GL}
	\trace (\bL)&=\sum_{i=1}^n \bL_{ii}   \nonumber \\
	&=\sum_{i=1}^{n} \lambda_i(\bL)  \nonumber \\
	&=\sum_{i=1}^n s_i  \nonumber \\
	&=s.	
	\end{align}
	
	By the PSD property of $\bL$, we have $\lambda_n(\bL) < s$ since $\lambda_2(\bL)>0$ for any connected graph.
	Therefore, by the orthogonality of eigenvectors of $\bL$ (i.e., $\bone_n^T \bv_i(\bL) =0$ for all $i \geq 2$) the eigenvalue decomposition of $\btL$ can be represented as 
	\begin{align}
	\btL&=\sum_{i=2}^n  \lambda_i (\bL) \bv_i(\bL) \bv_i^T(\bL)+\frac{s}{n} \bone_n \bone_n ^T \nonumber \\
	&=\sum_{i=1}^n  \lambda_i (\btL) \bv_i(\btL) \bv_i^T(\btL),
	\end{align}	
	where $(\lambda_n (\btL),\bv_n(\btL))=(s,\frac{\bone_n}{\sqrt{n}})$ and $(\lambda_i(\btL),\bv_i(\btL))=$
	$(\lambda_{i+1}(\bL),\bv_{i+1}(\bL))$ for $1 \leq i \leq n-1$.
\end{proof}


\begin{corollary}
	\label{cor_connect_NGL}
	For a normalized graph Laplacian matrix $\bLN$, assume $G$ is a connected graph and let $\btLN=\bLN+\frac{2}{s} \bS^{\frac{1}{2}} \bone_n \bone_n^T \bS^{\frac{1}{2}}$. Then  $(\lambda_i(\btLN),\bv_i(\btLN))=(\lambda_{i+1}(\bLN),
	\bv_{i+1}(\bLN))$ for $1 \leq i \leq n-1$ and $(\lambda_n(\btLN),\bv_n(\btLN))=(2,\frac{\bS^{\frac{1}{2}} \bone_n}{\sqrt{s}})$.
\end{corollary}
\begin{proof}
		Recall from (\ref{eqn_normalized_graph_Laplacian}) that $\bLN=\bS^{-\frac{1}{2}} \bL \bS^{-\frac{1}{2}}$, and also we have $\bLN \bS^{\frac{1}{2}} \bone_n=\bS^{-\frac{1}{2}} \bL \bone_n=\bzero_n$. Moreover, it can be shown that $0 \leq \lambda_1(\bLN) \leq \lambda_2(\bLN) \leq \ldots \leq \lambda_n(\bLN) \leq 2$ \cite{Merris94}, and 
	$\lambda_2(\bLN)>0$ if $G$ is connected. 
	Following the same derivation procedure for \textbf{Lemma} \ref{lem_connect_GL} we obtain the corollary. Note that $\bS^{\frac{1}{2}}=\diag(\sqrt{s_1},\sqrt{s_2},\ldots,\sqrt{s_n})$ and $( \bS^{\frac{1}{2}} \bone_n )^T \bS^{\frac{1}{2}} \bone_n=\bone_n^T \bS \bone_n=s$.
\end{proof}

\begin{lemma}
	\label{lem_disconnect_GL}
	Assume that $G$ is a disconnected graph with $\delta \geq 2$ connected components. Let $s=\sum_{i=1}^n s_i$, let $\bV=[\bv_1(\bL),\bv_2(\bL),\ldots,\bv_\delta(\bL)]$, and let 
	$\btL=\bL+ s \bV \bV^T$. 
	Then $(\lambda_i(\btL),\bv_i(\btL))=(\lambda_{i+\delta}(\bL),\bv_{i+\delta}(\bL))$ for $1 \leq i \leq n-\delta$,
	$\lambda_{i}(\btL)=s$ for $n-\delta+1 \leq i \leq n$, and 
	$[\bv_{n-\delta+1}(\btL),\bv_{n-\delta+2},(\btL),\ldots,\bv_{n}(\btL)]=\bV$.
\end{lemma}
\begin{proof}
	The graph Laplacian matrix of a disconnected graph consisting of $\delta$ connected components can be represented as a matrix with diagonal block structure, where each block in the diagonal corresponds to one connected component in $G$\cite{CPY13GlobalSIP}, that is,
\begin{align}
\label{eqn_L_disconnect}
\bL=
\left[
\begin{matrix}
\bL_1 & \bO     & \bO      & \bO  \\
\bO    & \bL_2 & \bO      & \bO \\
\bO     & \bO     & \ddots & \bO   \\
\bO     & \bO     & \bO      & \bL_\delta 
\end{matrix}
\right],
\end{align}
where $\bL_k$ is the graph Laplacian matrix of $k$-th connected component in $G$. From the proof of \textbf{Lemma} \ref{lem_connect_GL} each connected component contributes to exactly one zero eigenvalue for $\bL$, and 
\begin{align}
\lambda_n(\bL) &< \sum_{k=1}^{\delta} \sum_{i \in \text{component}~k} \lambda_i(\bL_k)
\nonumber \\
&=\sum_{k=1}^{\delta} \sum_{i \in \text{component}~k} s_i  \nonumber \\
&=s.
\end{align}
Therefore, we have the results in \textbf{Lemma} \ref{lem_disconnect_GL}.
\end{proof}

\textbf{Lemma} \ref{lem_connect_GL} applies to the (unnormalized) graph Laplacian matrix of a connected graph, and
the corollary below applies to the normalized graph Laplacian matrix of a connected graph.

\begin{corollary}
	\label{cor_disconnect_NGL}
	For a normalized graph Laplacian matrix $\bLN$, assume $G$ is a disconnected graph with $\delta \geq 2$ connected components. Let 
	$\bV_\delta=[\bv_1(\bLN),\bv_2(\bLN), \\ \ldots,\bv_\delta(\bLN)]$, and let 
	$\btLN=\bLN+ 2 \bV_\delta \bV_{\delta}^T.$ 
	Then $(\lambda_i(\btLN),\bv_i(\btLN))=(\lambda_{i+\delta}(\bLN),\\ \bv_{i+\delta}(\bLN))$ for $1 \leq i \leq n-\delta$,
	$\lambda_{i}(\btLN)=2$ for $n-\delta+1 \leq i \leq n$, and 
	$[\bv_{n-\delta+1}(\btLN),\bv_{n-\delta+2},(\btLN),\ldots,\bv_{n}(\btLN)]=\bV_\delta$.
\end{corollary}
\begin{proof}
		The results can be obtained by following the same derivation proof procedure as in
 Lemma \ref{lem_disconnect_GL} and the fact that 
	$\lambda_n(\bLN) \leq 2$ \cite{Merris94}.
\end{proof}

\begin{rem}
	note that the columns of any matrix $\bV^\prime= \bV \bR$ with an orthonormal transformation matrix $\bR$ (i.e., $\bR^T \bR=\bI$) are also the largest $\delta$ eigenvectors of $\btL$ and $\btLN$ in \textbf{Lemma} \ref{lem_disconnect_GL} and \textbf{Corollary} \ref{cor_disconnect_NGL}. Without loss of generality we consider the case $\bR=\bI$.
\end{rem}

\subsection{Incremental method of increasing orders (Incremental-IO)}
\label{subsec_incremental_method}
Given the $K$ smallest eigenpairs of a graph Laplacian matrix, we prove that
computing the $(K+1)$-th smallest eigenpair is equivalent to computing the
leading eigenpair (the eigenpair with the largest eigenvalue in absolute value)
of a certain perturbed matrix. The advantage of this
transformation is that the leading eigenpair can be efficiently computed via
matrix power iteration methods \cite{lanczos1950iteration,lehoucq1998arpack}.  

Let $\bV_{K}=[\bv_1{(\bL)},\bv_2{(\bL)},\ldots,\bv_K{(\bL)}]$ be the matrix
with columns being the $K$ smallest eigenvectors of $\bL$ and let
$\bLambda_{K}=\diag
(s-\lambda_1(\bL),s-\lambda_2(\bL),\ldots,s-\lambda_K(\bL))$ be the diagonal
matrix with $\{s-\lambda_i(\bL)\}_{i=1}^K$ being its main diagonal. The
following theorems show that given the $K$ smallest eigenpairs of $\bL$, the
$(K+1)$-th smallest eigenpair of $\bL$ is the leading eigenvector of the
original graph Laplacian matrix perturbed by a certain matrix.

\begin{theorem}{\textnormal{(connected graphs)}} 
	\label{thm_connect_GL}
	Given $\bV_{K}$ and $\bLambda_{K}$, assume that $G$ is a connected graph. Then the eigenpair $(\lambda_{K+1} (\bL),\bv_{K+1}(\bL))$ is a leading eigenpair of the matrix $\btL=\bL+\bV_{K} \bLambda_K \bV_K^T +\frac{s}{n} \bone_n \bone_n^T - s \bI$. In particular, if $\bL$ has distinct eigenvalues, then $(\lambda_{K+1}(\bL),\bv_{K+1}(\bL)) 
	=(\lambda_{1}(\btL)+s,\bv_{1}(\btL))$, and $\lambda_1(\btL)$ is the largest eigenvalue of $\btL$ in magnitude.
\end{theorem}
\begin{proof}
		From  \textbf{Lemma} \ref{lem_connect_GL},
	\begin{align}
	\label{eqn_appex_thm_connect_GL_1}
	&\bL+\frac{s}{n} \bone_n \bone_n^T+\bV_{K} \bLambda_K \bV_K^T \nonumber \\
	&=\sum_{i=K+1}^n \lambda_i(\bL) \bv_i(\bL)  \bv_i^T(\bL)+ \sum_{i=2}^K s \cdot \bv_i(\bL)  \bv_i^T(\bL)+\frac{s}{n} \bone_n \bone_n^T,
	\end{align}
	which is a valid eigenpair decomposition that can be seen by inflating the $K$ smallest eigenvalues of $\bL$ to $s$ with the originally paired eigenvectors. Using (\ref{eqn_appex_thm_connect_GL_1}) we obtain the eigenvalue decomposition of $\btL$ as
	\begin{align}
	\label{eqn_appex_thm_connect_GL_2}
	\btL&=\bL+\bV_{K} \bLambda_K \bV_K^T +\frac{s}{n} \bone_n \bone_n^T - s \bI \nonumber \\
	&=\sum_{i=K+1}^n (\lambda_i(\bL)-s) \bv_i(\bL)  \bv_i^T(\bL),
	\end{align}
	where we obtain the eigenvalue relation $\lambda_i(\btL)=\lambda_{i+K}(\bL)-s$.
	Furthermore, since $0 \leq \lambda_{K+1}(\bL) \leq \lambda_{K+2}(\bL) \leq \ldots \leq \lambda_{n}(\bL)$, we have 
	$ |\lambda_1(\btL)|=|\lambda_{K+1}(\bL)-s| \geq |\lambda_{K+2}(\bL)-s| \geq \ldots \geq |\lambda_{n}(\bL)-s|=|\lambda_{n-K}(\btL)|$. Therefore, the eigenpair $(\lambda_{K+1}(\bL),\bv_{K+1}(\bL))$ can be obtained by computing the leading eigenpair of $\btL$. In particular, if $\bL$ has distinct eigenvalues, then the leading eigenpair of $\btL$ is unique. Therefore, by (\ref{eqn_appex_thm_connect_GL_2}) we have the relation
	\begin{align}
	(\lambda_{K+1}(\bL),\bv_{K+1}(\bL))=(\lambda_{1}(\btL)+s,\bv_{1}(\btL)).
	\end{align}
\end{proof}

The next theorem describes an incremental eigenpair computation method when the graph $G$ is a disconnected graph with $\delta$ connected components.  The columns of the matrix $\bV_\delta$ are the $\delta$ smallest eigenvectors of $\bL$. Note that $\bV_\delta$ has a canonical representation that the nonzero entries of each column are a constant and their indices indicate the nodes in each connected component \cite{Luxburg07,CPY13GlobalSIP}, and the columns of $\bV_\delta$ are the $\delta$ smallest eigenvectors of $\bL$ with eigenvalue $0$ \cite{CPY13GlobalSIP}.
Since the $\delta$ smallest eigenpairs with the canonical representation are trivial by identifying the connected components in the graph,  we only focus on computing the $(K+1)$-th smallest eigenpair given $K$ smallest eigenpairs, where $K \geq \delta$. 
The columns of the matrix $\bV_{K,\delta}=[\bv_{\delta+1}{(\bL)},\bv_{\delta+2}{(\bL)},\ldots,\bv_K{(\bL)}]$ are the $(\delta+1)$-th to the $K$-th smallest eigenvectors of $\bL$ and the matrix $\bLambda_{K,\delta}=\diag (s-\lambda_{\delta+1}(\bL),s-\lambda_{\delta+2}(\bL),\ldots,s-\lambda_K(\bL))$ is the diagonal matrix with $\{s-\lambda_i(\bL)\}_{i=\delta+1}^K$ being its main diagonal. If $K=\delta$, $\bV_{K,\delta}$ and $\bLambda_{K,\delta}$ are defined as the matrix with all entries being zero, i.e., $\bO$.
\begin{theorem}{\textnormal{(disconnected graphs)}} 
	\label{thm_disconnect_GL}
	Assume that $G$ is a disconnected graph with $\delta \geq 2$ connected components, given $\bV_{K,\delta}$, $\bLambda_{K,\delta}$ and $K \geq \delta$, the eigenpair $(\lambda_{K+1} (\bL),\bv_{K+1}(\bL))$ is a leading eigenpair of the matrix $\btL=\bL+\bV_{K,\delta} \bLambda_{K,\delta} \bV_{K,\delta}^T \\ + s \bV_\delta \bV_{\delta}^T - s \bI$. In particular, if $\bL$ has distinct nonzero eigenvalues, then \\ $(\lambda_{K+1}(\bL),\bv_{K+1}(\bL))=(\lambda_{1}(\btL)+s,\bv_{1}(\btL))$, and $\lambda_1(\btL)$ is the largest eigenvalue of $\btL$ in magnitude.
\end{theorem}
\begin{proof}
	First observe from (\ref{eqn_L_disconnect}) that $\bL$ has $\delta$ zero eigenvalues since each connected component contributes to exactly one zero eigenvalue for $\bL$.
	Following the same derivation procedure in the proof of \textbf{Theorem} \ref{thm_connect_GL} and using  \textbf{Lemma} \ref{lem_disconnect_GL}, we have
	\begin{align}
	\label{eqn_appex_thm_disconnect_GL_1}
	\btL&=\bL+\bV_{K,\delta} \bLambda_{K,\delta} \bV_{K,\delta}^T +s \bV_\delta \bV_{\delta}^T - s \bI \nonumber \\
	&=\sum_{i=K+1,K \geq \delta}^n (\lambda_i(\bL)-s) \bv_i(\bL)  \bv_i^T(\bL).
	\end{align}
	Therefore, the eigenpair $(\lambda_{K+1}(\bL),\bv_{K+1}(\bL))$ can be obtained by computing the leading eigenpair of $\btL$.
	If $\bL$ has distinct nonzero eigenvalues (i.e, $\lambda_{\delta+1}(\bL) < \lambda_{\delta+2}(\bL) < \ldots < \lambda_{n}(\bL) $), we obtain the relation 
	$ (\lambda_{K+1}(\bL),\bv_{K+1}(\bL))=(\lambda_{1}(\btL)+s,\bv_{1}(\btL))$.
\end{proof}

Following the same methodology for proving \textbf{Theorem} \ref{thm_connect_GL} and using \textbf{Corollary} \ref{cor_connect_NGL}, for normalized graph Laplacian matrices, let  $\bV_{K}=[\bv_1{(\bLN)},\bv_2{(\bLN)}, \\ \ldots,\bv_K{(\bLN)}]$ be the matrix with columns being the $K$ smallest eigenvectors of $\bLN$ and let $\bLambda_{K}=\diag (2-\lambda_1(\bLN),2-\lambda_2(\bLN),\ldots,2-\lambda_K(\bLN))$. The following corollary provides the basis for incremental eigenpair computation for normalized graph Laplacian matrix of connected graphs.

\begin{corollary}
	\label{cor_connect_NGL_incremental}
	For the normalized graph Laplacian matrix $\bLN$ of a connected graph $G$, given $\bV_{K}$ and $\bLambda_{K}$, the eigenpair $(\lambda_{K+1} (\bLN),\bv_{K+1}(\bLN))$ is a leading eigenpair of the matrix $\btLN=\bLN+\bV_{K} \bLambda_K \bV_K^T +\frac{2}{s} \bS^{\frac{1}{2}} \bone_n \bone_n^T \bS^{\frac{1}{2}} - 2 \bI$. In particular, if $\bLN$ has distinct eigenvalues, then $(\lambda_{K+1}(\bLN),$ 
	$\bv_{K+1}(\bLN))=(\lambda_{1}(\btLN)+2,\bv_{1}(\btLN))$, and $\lambda_1(\btLN)$ is the largest eigenvalue of $\btLN$ in magnitude.
\end{corollary}
\begin{proof}
	The proof procedure is similar to the proof of \textbf{Theorem} \ref{thm_connect_GL} by using \textbf{Corollary} \ref{cor_connect_NGL}.
\end{proof}

For disconnected graphs with $\delta$ connected components,  
let  $\bV_{K,\delta}=[\bv_{\delta+1}{(\bLN)}, \\ \bv_{\delta+2}{(\bLN)},\ldots,\bv_K{(\bLN)}]$ with columns being the $(\delta+1)$-th to the $K$-th smallest eigenvectors of $\bLN$ and let $\bLambda_{K,\delta}=\diag (2-\lambda_{\delta+1}(\bLN),2-\lambda_{\delta+2}(\bLN),\ldots,2-\lambda_K(\bLN))$.
Based on \textbf{Corollary} \ref{cor_disconnect_NGL}, the following corollary provides an incremental eigenpair computation method for normalized graph Laplacian matrix of disconnected graphs.

\begin{corollary}
	\label{cor_disconnect_NGL_incremental}
	For the normalized graph Laplacian matrix $\bLN$ of a disconnected graph $G$ with $\delta \geq 2$ connected components, given $\bV_{K,\delta}$, $\bLambda_{K,\delta}$ and $K \geq \delta$,
	the eigenpair $(\lambda_{K+1} (\bLN),\bv_{K+1}(\bLN))$ is a leading eigenpair of the matrix $\btLN=\bLN+\bV_{K,\delta} \bLambda_{K,\delta} \bV_{K,\delta}^T +\frac{2}{s} \bS^{\frac{1}{2}} \bone_n \bone_n^T \bS^{\frac{1}{2}} - 2 \bI$. In particular, if $\bLN$ has distinct eigenvalues, then $(\lambda_{K+1}(\bLN),\bv_{K+1}(\bLN))$
	$=(\lambda_{1}(\btLN)+2,\bv_{1}(\btLN))$, and $\lambda_1(\btLN)$ is the largest eigenvalue of $\btLN$ in magnitude.
\end{corollary}
\begin{proof}
	The proof procedure is similar to the proof of \textbf{Theorem} \ref{thm_disconnect_GL} by using \textbf{Corollary} \ref{cor_disconnect_NGL}.
\end{proof}

\subsection{Computational complexity analysis}
\label{subsec_computation_analysis}
Here we analyze the computational complexity of Incremental-IO and compare it with the batch computation method.
Incremental-IO utilizes the existing $K$ smallest eigenpairs
to compute the $(K+1)$-th smallest eigenpair as described in Sec.
\ref{subsec_incremental_method}, whereas the batch computation method 
recomputes all $K$ smallest eigenpairs for each value of $K$. Both methods can
be easily implemented via well-developed numerical computation packages such as
ARPACK \cite{lehoucq1998arpack}.

Following the analysis in \cite{kuczynski1992estimating}, the average relative
error of the leading eigenvalue from the Lanczos algorithm
\cite{lanczos1950iteration} has an upper bound of the order   $O \lb \frac{(\ln
	n)^2}{t^2} \rb$, where $n$ is the number of nodes in the graph and $t$ is the
number of iterations for Lanczos algorithm. Therefore, when one sequentially
computes from $k=2$ to $k=K$ smallest eigenpairs, for Incremental-IO the upper bound on the average relative error of $K$
smallest eigenpairs is  $O \lb \frac{K (\ln n)^2}{t^2} \rb$ since in each
increment computing the corresponding eigenpair can be transformed to a leading
eigenpair computation process as described in Sec. \ref{subsec_incremental_method}. On the other hand, for the batch computation
method, the upper bound on the average relative error of $K$ smallest
eigenpairs is  $O \lb \frac{K^2 (\ln n)^2}{t^2} \rb$ since for the $k$-th
increment ($k \leq K$) it needs to compute all $k$ smallest eigenpairs from
scratch. These results also imply that to reach the same average relative error
$\epsilon$ for sequential computation of $K$ smallest eigenpairs, Incremental-IO requires $\Omega \lb \sqrt{\frac{K}{\epsilon}}  \ln n  \rb$
iterations, whereas the batch method requires $\Omega \lb \frac{K\ln
	n}{\sqrt{{\epsilon}}} \rb$  iterations. It is difficult to analyze the computational complexity of Lanczos-IO, as its convergence results heavily depend on the quality of previously generated Lanczos vectors.

\subsection{Incremental-IO via power iterations}

Using the theoretical results of Incremental-IO developed in Theorem \ref{thm_connect_GL} and Corollary \ref{cor_connect_NGL_incremental}, we illustrate how Incremental-IO can be implemented via power iterations for connected graphs, which is summarized in Algorithm \ref{algo_incremental_connected}. The case of disconnected graphs can be implemented in a similar way using Theorem \ref{thm_disconnect_GL} and Corollary \ref{cor_disconnect_NGL_incremental}.

\begin{algorithm}[t]
	\caption{Incremental-IO via power iterations for connected graphs}
	\label{algo_incremental_connected}
	\begin{algorithmic}
		\State \textbf{Input:} $K$ smallest eigenpairs $\{\lambda_k,\bv_k\}_{k=1}^K$ of $\bL$ (or $\bLN$), \# of power iterations $h$, sum of nodal strength $s$, diagonal nodal strength matrix $\bS$
		\State \textbf{Output:} $(K+1)$-th smallest eigenpair  $(\lambda_{K+1},\bv_{K+1})$
		\State \textbf{Initialization:} a random vector $\bx$  in $\mathbb{R}^n$ with unit norm $\|\bx\|_2=1$. 
		\For{$i=1$ to $h$}
		\State (Unnormalized graph Laplacian matrix $\bL$)
		\State U1. $\by = \bL \bx + \sum_{k=1}^K \lambda_k (\bv_k^T \bx) \bv_k + \frac{s}{n} (\bone_n^T \bx)  \bone_n - s \bx$.
		\State U2. $\bx = \frac{\by}{\|\by\|_2}$.
		\State (Normalized graph Laplacian matrix $\bLN$)	
		\State N1.  $\by = \bLN \bx + \sum_{k=1}^K \lambda_k (\bv_k^T \bx) \bv_k +\frac{2}{s} (\bone_n^T \bS^{\frac{1}{2}} \bx) \bS^{\frac{1}{2}} \bone_n- 2 \bx$.	
		\State N2.  $\bx = \frac{\by}{\|\by\|_2}$.
		\EndFor	
		\State (Unnormalized graph Laplacian matrix $\bL$)				
		\State U3. $(\lambda_{K+1},\bv_{K+1})=(s+\bx^T \by,\bx)$.
		\State (Normalized graph Laplacian matrix $\bLN$)		
		\State N3. $(\lambda_{K+1},\bv_{K+1})=(2+\bx^T \by,\bx)$.		
	\end{algorithmic}
\end{algorithm}

In essence, since in Theorem \ref{thm_connect_GL} we proved that  given the $K$ smallest eigenpairs of $\bL$, the leading eigenvector of $\btL$ is the $(K+1)$-th smallest eigenvector of the unnormalized graph Laplacian matrix $\bL$, one can apply the Rayleigh quotient iteration method \cite{HornMatrixAnalysis} to approximate the leading eigenpairs\ of $\btL$ via power iterations. The same argument applies to normalized graph Laplacian matrix $\bLN$ via Corollary \ref{cor_connect_NGL_incremental}.
 In particular, $\bL$ (or $\bLN$) has $m+n$ nonzero entries, where $n$ and $m$ are the number of nodes and edges in the graph respectively.
Therefore, in Algorithm \ref{algo_incremental_connected} the implementation of $h$ power iterations requires $O(h(m+Kn))$ operations, and the data storage requires $O(m+Kn)$ space.

\section{Application: User-Guided Spectral Clustering with Incremental-IO }
Based on the developed incremental eigenpair computation method (Incremental-IO) in Sec. \ref{sec_incremental_eigenpair}, we propose an
incremental algorithm for user-guided spectral clustering as summarized in
\textbf{Algorithm} \ref{algo_incremental_automated_clustering}. This algorithm
sequentially computes the smallest eigenpairs via Incremental-IO (steps 1-3) for spectral clustering and provides
a sequence of clusters with the values of user-specified clustering metrics.

The input graph is a connected undirected weighted  graph $\bW$ and we convert
it to the reduced weighted graph $\bWN=\bS^{-\frac{1}{2}} \bW
\bS^{-\frac{1}{2}}$ to alleviate the effect of imbalanced edge weights. The
entries of $\bWN$ are properly normalized by the nodal strength such that
$[\bWN]_{ij}=\frac{[\bW]_{ij}}{\sqrt{s_i \cdot s_j}}$. We then obtain the graph
Laplacian matrix $\bL$ for $\bWN$ and incrementally compute the eigenpairs of
$\bL$ via Incremental-IO (steps 1-3)
until the user decides to stop further computation.

Starting from $K=2$ clusters, the algorithm incrementally computes the $K$-th
smallest eigenpair $(\lambda_{K}(\bL),\bv_{K}(\bL))$ of $\bL$ with the
knowledge of the previous $K-1$ smallest eigenpairs via \textbf{Theorem}
\ref{thm_connect_GL} and obtains matrix $\bV_K$ containing $K$ smallest
eigenvectors. By viewing each row in $\bV_K$ as a $K$-dimensional vector,
K-means clustering is implemented to separate the rows in $\bV_K$ into $K$
clusters. For each increment, the identified $K$ clusters are denoted by
$\{\hG_k\}_{k=1}^K$, where $\hG_k$ is a graph partition with $\hn_k$ nodes and
$\hatm_k$ edges.

\begin{algorithm}[t]
	\caption{Incremental algorithm for user-guided spectral clustering using Incremental-IO (steps 1-3)}
	\label{algo_incremental_automated_clustering}
	\begin{algorithmic}
		\State \textbf{Input:}  connected undirected weighted graph $\bW$, user-specified clustering metrics  
		\State \textbf{Output:} $K$ clusters $\{\hG_k\}_{k=1}^K$
		\State \textbf{Initialization:} $K=2$. $\bV_1=\bLambda_1=\bO$. Flag $=1$. $\bS=\diag(\bW \bone_n)$. $\bWN=\bS^{-\frac{1}{2}} \bW \bS^{-\frac{1}{2}}$. 
		\State ~~~~~~~~~~~~~~~~~~~~~$\bL=\diag(\bWN \bone_n)-\bWN$. $s=\bone_n^T \bWN \bone_n$.
		\While{Flag$=1$}
		\State 1. $\btL=\bL+\bV_{K-1} \bLambda_{K-1} \bV_{K-1}^T +\frac{s}{n} \bone_n \bone_n^T - s \bI$.
		\State 2. Compute the leading eigenpair $(\lambda_{1}(\btL),\bv_{1}(\btL))$ and set
		\State  ~~~~$(\lambda_{K}(\bL),\bv_{K}(\bL))=(\lambda_{1}(\btL)+s,\bv_{1}(\btL))$.
		\State 3. Update $K$ smallest eigenpairs of $\bL$ by $\bV_K=[\bV_{K-1}~\bv_K]$  
		\State ~~~~and $[\bLambda_K]_{KK}=s-\lambda_{K}(\bL)$.
		\State 4. Perform K-means clustering on the rows of $\bV_K$
		to obtain $K$ clusters $\{\hG_k\}_{k=1}^K$.
		
		\State 5. Compute user-specified clustering metrics.
		\If{user decides to stop}
		Flag$=0$ 
		\State Output  $K$ clusters $\{\hG_k\}_{k=1}^K$
		\Else
		\State Go back to step 1 with $K=K+1$.
		\EndIf
		\EndWhile
	\end{algorithmic}
\end{algorithm}

In addition to incremental computation of smallest eigenpairs, for each
increment the algorithm can also be used to update clustering metrics 
such as normalized cut, modularity, and cluster size distribution, in order to
provide users with clustering information to stop the incremental computation
process. The incremental computation algorithm allows users to
efficiently track the changes in clusters as the number $K$ of hypothesized clusters increases.

Note that \textbf{Algorithm} \ref{algo_incremental_automated_clustering} is
proposed for connected graphs and their corresponding unnormalized graph
Laplacian matrices. The algorithm can be easily adapted to disconnected graphs
or normalized graph Laplacian matrices by modifying steps 1-3 based on the developed results in
\textbf{Theorem} \ref{thm_disconnect_GL}, \textbf{Corollary}
\ref{cor_connect_NGL_incremental} and \textbf{Corollary}
\ref{cor_disconnect_NGL_incremental}.

\section{Implementation}

We implement the proposed incremental eigenpair computation method \\(Incremental-IO) using Matlab
R2015a's ``eigs'' function, which is based on ARPACK package
\cite{lehoucq1998arpack}. Note that this function takes a parameter $K$ and
returns $K$ leading eigenpairs of the given matrix.
The $eigs$ function is
implemented in Matlab with a Lanczos algorithm that computes the leading eigenpairs
(the implicitly-restarted Lanczos method \cite{calvetti1994implicitly}). This
Matlab function iteratively generates Lanczos vectors starting from an initial
vector (the default setting is a random vector) with restart.  Following
\textbf{Theorem} \ref{thm_connect_GL}, Incremental-IO works by
sequentially perturbing the graph Laplacian matrix $\bL$ with a particular
matrix and computing the leading eigenpair of the perturbed matrix $\btL$ (see
\textbf{Algorithm} \ref{algo_incremental_automated_clustering}) by calling
$eigs(\btL, 1)$.   For fair comparison to other two methods (Lanczos-IO and the batch computation method), we select the $eigs$ function over the power iteration method for implementing Incremental-IO, as the latter method provides approximate eigenpair computation and its approximation error depends on the number of power iterations, while the former guarantees $\epsilon$-accuracy for each compared method.
For the batch
computation method, we use $eigs(\bL, K)$ to compute the desired $K$ eigenpairs
from scratch as $K$ increases. 

\begin{algorithm}[t]
	\caption{Lanczos method of Increasing Orders (Lanczos-IO)}
	\label{algo_LanczosIO}
	\begin{algorithmic}
		\State \textbf{Input:} real symmetric matrix $\bM$, $\#$ of initial Lanczos vectors $Z_{ini}$, $\#$ of augmented Lanczos vectors $Z_{aug}$ 
		\State \textbf{Output:} $K$ leading eigenpairs $\{\lambda_i,\bv_i\}_{i=1}^K$ of $\bM$
		\State \textbf{Initialization:} Compute $Z_{ini}$ Lanczos vectors as columns of $\bQ$ and the corresponding tridiagonal matrix $\bT$ of $\bM$. Flag $=1$.  $K=1$. $Z=Z_{ini}$.
		
		\While{Flag$=1$}		
		\State 1. Obtain the $K$ leading eigenpairs $\{t_i,\bu_i\}_{i=1}^K$ of $\bT$. $\bU=[\bu_1,\ldots,\bu_K]$.
		\State 2. Residual error $=|\bT(Z-1,Z) \cdot \bU(Z,K)|$
		\While{Residual error $>$ Tolerance}	
		\State 2-1. $Z=Z+Z_{aug}$
		\State 2-2. Based on $\bQ$ and $\bT$, compute the next $Z_{aug}$ Lanczos vectors as columns  of
		\State ~~~~~~$\bQ_{aug}$ and  the augmented tridiagonal matrix $\bT_{aug}$
		\State 2-3. $\bQ \leftarrow [\bQ~\bQ_{aug}]$ and $\bT \leftarrow \begin{bmatrix}
		\bT & \bO \\
		\bO  & \bT_{aug}
		\end{bmatrix}$ 
		\State 2-4. Go back to step 1
		\EndWhile 
		\State 3. $\{\lambda_i\}_{i=1}^K=\{t_i\}_{i=1}^K$.~ $[\bv_1,\ldots,\bv_K]=\bQ \bU$.
		\If{user decides to stop}
		Flag$=0$ 
		\State Output  $K$ leading eigenpairs $\{\lambda_i,\bv_i\}_{i=1}^K$
		\Else
		\State Go back to step 1 with $K=K+1$.
		\EndIf
		\EndWhile
	\end{algorithmic}
\end{algorithm}

For implementing Lanczos-IO, we extend the Lanczos algorithm of fixed order
($K$ is fixed) using the PROPACK package \cite{larsen2000computing}.
As we have stated earlier, Lanczos-IO works by storing all previously generated
Lanczos vectors and using them to compute new Lanczos vectors for each
increment in $K$. The general procedure of computing $K$ leading
eigenpairs of a real symmetric matrix $\bM$ using Lanczos-IO is described in
\textbf{Algorithm} \ref{algo_LanczosIO}.  The operation of Lanczos-IO is
similar to the explicitly-restarted Lanczos algorithm \cite{wu2000thick}, which
restarts the computation of Lanczos vectors with a subset of previously
computed Lanczos vectors.  Note that the Lanczos-IO consumes additional memory
for storing all previously computed Lanczos vectors when compared with the
proposed incremental method in \textbf{Algorithm}
\ref{algo_incremental_automated_clustering}, since the $eigs$ function uses the
implicitly-restarted Lanczos method that spares the need of storing Lanczos
vectors for restart.

To apply Lanczos-IO to spectral clustering of increasing orders, we can set
$\bM=\bL+\frac{s}{n} \bone_n \bone_n^T - s \bI$ to obtain the smallest
eigenvectors of $\bL$.   Throughout the experiments the parameters in
\textbf{Algorithm} \ref{algo_LanczosIO} are set to be $Z_{ini}=20$ and
$\textnormal{Tolerence}=\epsilon \cdot \|\bM \|$, where $\epsilon$ is the
machine precision, $\|\bM\|$ is the operator norm of $\bM$, and these settings
are consistent with the settings used in $eigs$ function \cite{lehoucq1998arpack}.
The number of augmented Lanczos vectors $Z_{aug}$ is set to be $10$, and the effect of $Z_{aug}$  on the computation time is discussed in Sec. \ref{subsec_time_real}.
The Matlab implementation of the aforementioned batch method, Lanczos-IO, and Incremental-IO are
available from the first author's personal website: https://sites.google.com/site/pinyuchenpage/codes

\section{Experimental Results}
\label{sec:ex}
In this section we perform several experiments: first, compare the computation time between
Incremental-IO, Lanczos-IO, and the batch method; second, numerically verify the accuracy of Incremental-IO;
third, demonstrate the
usages of Incremental-IO for user-guided spectral
clustering. For the first experiment, we generate synthetic Erdos-Renyi random
graphs of various sizes. For the second experiment, we compare the consistency of eigenpairs obtained from Incremental-IO and the batch method.
For the third experiment, we use six popular graph
datasets as listed in Table \ref{tab:statistic}. The descriptions of these datasets are as follows.

\begin{enumerate} 

\item Minnesota Road
Map dataset is a graphical representation of roads in Minnesota, where nodes
represent road intersections and edges represent roads. 

\item Power Grid is a graph
representing the topology of Western Power Grid of USA, where nodes represent
power stations and edges represent power lines.

\item CLUTO is a synthetic
dataset of two-dimensional data points for density-based clustering.

\item Swiss Roll is a synthetic dataset designed
for manifold learning tasks. The data points lies in a two-dimensional manifold of a
three-dimensional space.

\item Youtube is a graph representing  social interactions of users on Youtube, where nodes are users and
edges are existence of interactions.

\item BlogCatalog is a social friendship graph among bloggers, where nodes represent bloggers and edges
represent their social interactions.

\end{enumerate}

\begin{table}[]
	\centering	
	\caption{Statistics of Datasets}
	\label{tab:statistic}
	\begin{tabular}{llll}
		\hline
		Dataset                                                   & Nodes                                                          & Edges                                                          & Density \\ \hline
		\begin{tabular}[c]{@{}l@{}}Minnesota \\ Road Map\footnotemark\end{tabular} & \begin{tabular}[c]{@{}l@{}}2640 \\ intersections\end{tabular}  & \begin{tabular}[c]{@{}l@{}}3302 \\ roads\end{tabular}          & 0.095\% \\ \hline
		\begin{tabular}[c]{@{}l@{}}Power \\ Grid\footnotemark\end{tabular}     & \begin{tabular}[c]{@{}l@{}}4941 \\ power stations\end{tabular} & \begin{tabular}[c]{@{}l@{}}6594 \\ power lines\end{tabular}    & 0.054\% \\ \hline
		CLUTO\footnotemark                                                     & \begin{tabular}[c]{@{}l@{}}7674 \\ data points\end{tabular}    & \begin{tabular}[c]{@{}l@{}}748718 \\ kNN edges\end{tabular}    & 2.54\%  \\ \hline
		\begin{tabular}[c]{@{}l@{}}Swiss \\ Roll\footnotemark\end{tabular}     & \begin{tabular}[c]{@{}l@{}}20000 \\ data points\end{tabular}   & \begin{tabular}[c]{@{}l@{}}81668 \\ kNN edges\end{tabular}     & 0.041\% \\ \hline
		Youtube\footnotemark                                                   & \begin{tabular}[c]{@{}l@{}}13679 \\ users\end{tabular}         & \begin{tabular}[c]{@{}l@{}}76741 \\ interactions\end{tabular}  & 0.082\% \\ \hline
		BlogCatalog\footnotemark                                               & \begin{tabular}[c]{@{}l@{}}10312 \\ bloggers\end{tabular}      & \begin{tabular}[c]{@{}l@{}}333983 \\ interactions\end{tabular} & 0.63\%  \\ \hline
	\end{tabular}
\end{table}

\footnotetext[1]{{\scriptsize http://www.cs.purdue.edu/homes/dgleich/nmcomp/matlab/minnesota}}
\footnotetext[2]{{\scriptsize http://www-personal.umich.edu/~mejn/netdata}}
\footnotetext[3]{{\scriptsize http://glaros.dtc.umn.edu/gkhome/views/cluto}}
\footnotetext[4]{{\scriptsize http://isomap.stanford.edu/datasets.html}}
\footnotetext[5]{{\scriptsize http://socialcomputing.asu.edu/datasets/YouTube}}
\footnotetext[6]{{\scriptsize http://socialcomputing.asu.edu/datasets/BlogCatalog}}

Among all six datasets, Minnesota Road Map
and Power Grid are unweighted graphs, and the others are weighted graphs. For
CLUTO and Swiss Roll we use $k$-nearest neighbor (KNN) algorithm to generate
similarity graphs, where the parameter $k$ is the minimal value that makes the
graph connected, and the similarity is measured by Gaussian kernel with unity
bandwidth \cite{Zaki.Jr:14}.

\subsection{Comparison of computation time on simulated	graphs}
To illustrate the advantage of Incremental-IO, we compare its computation time with
the other two methods, the batch method and Lanczos-IO, for varying order $K$ and varying graph size $n$.  The Erdos-Renyi
random graphs that we build are used for this comparison. Fig.
\ref{Fig_incremental_computation} (a) shows the computation time of
Incremental-IO, Lanczos-IO, and the batch computation method for sequentially
computing from $K=2$ to $K=10$ smallest eigenpairs.  It is observed that the
computation time of Incremental-IO and Lanczos-IO grows linearly as $K$
increases, whereas the computation time of the batch method grows superlinearly
with $K$. 

Fig.  \ref{Fig_incremental_computation} (b) shows the computation time of all three
methods with respect to different graph size $n$.  It is observed that the
difference in computation time between the batch method and the two incremental
methods grow polynomially as $n$ increases, which suggests that in this experiment Incremental-IO and Lanczos-IO are more efficient than the
batch computation method, especially for large graphs. It is worth noting that
although Lanczos-IO has similar performance in computation time as Incremental-IO, it requires additional memory allocation for storing all
previously computed Lanczos vectors.

\begin{figure*}[t]
	\centering
	\begin{subfigure}[b]{0.5\textwidth}
		\includegraphics[width=\textwidth]{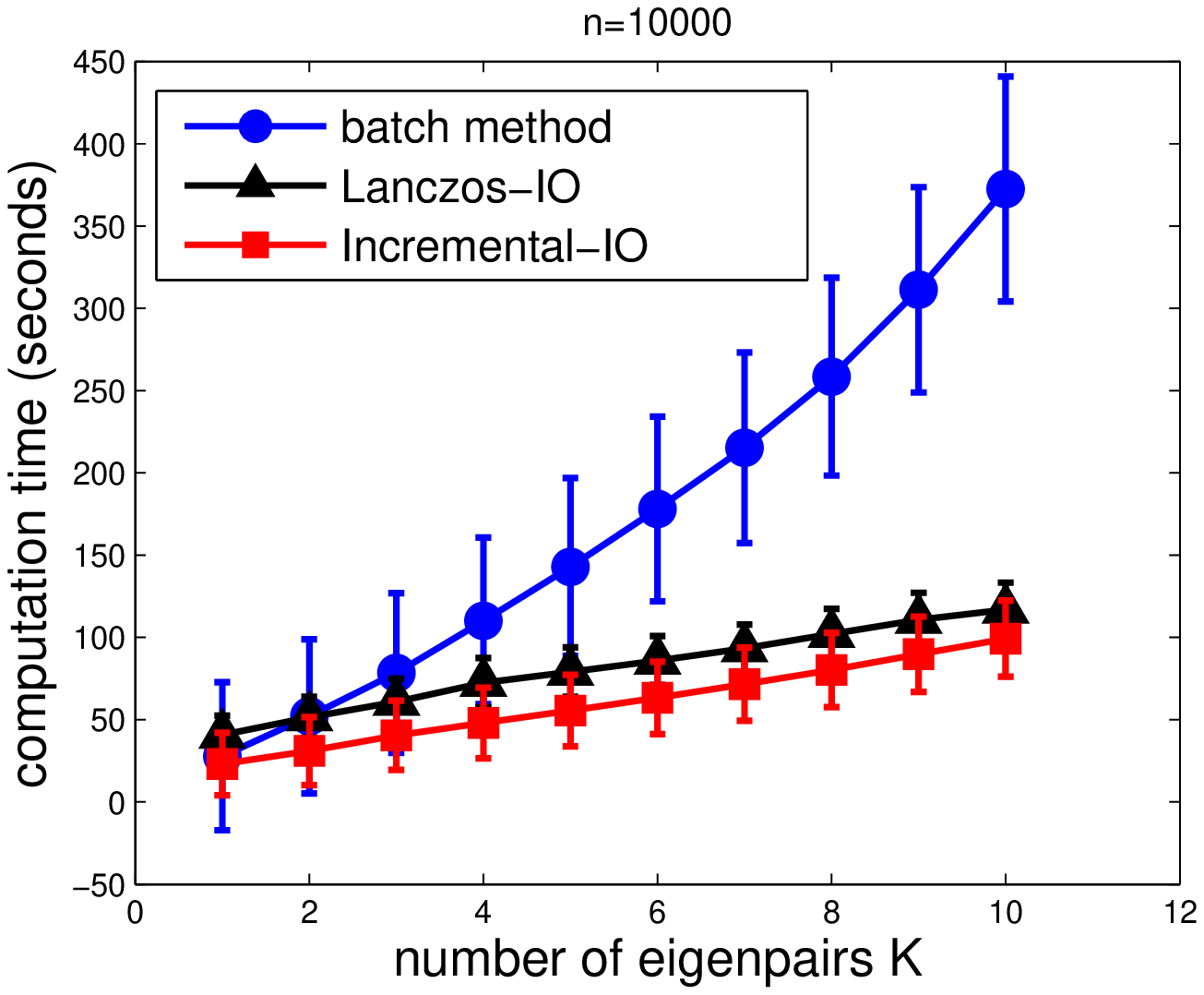}
		\caption{~}
		\label{Fig_incremental_computation_K}
	\end{subfigure}%
	\centering
	\begin{subfigure}[b]{0.5\textwidth}
		\includegraphics[width=\textwidth]{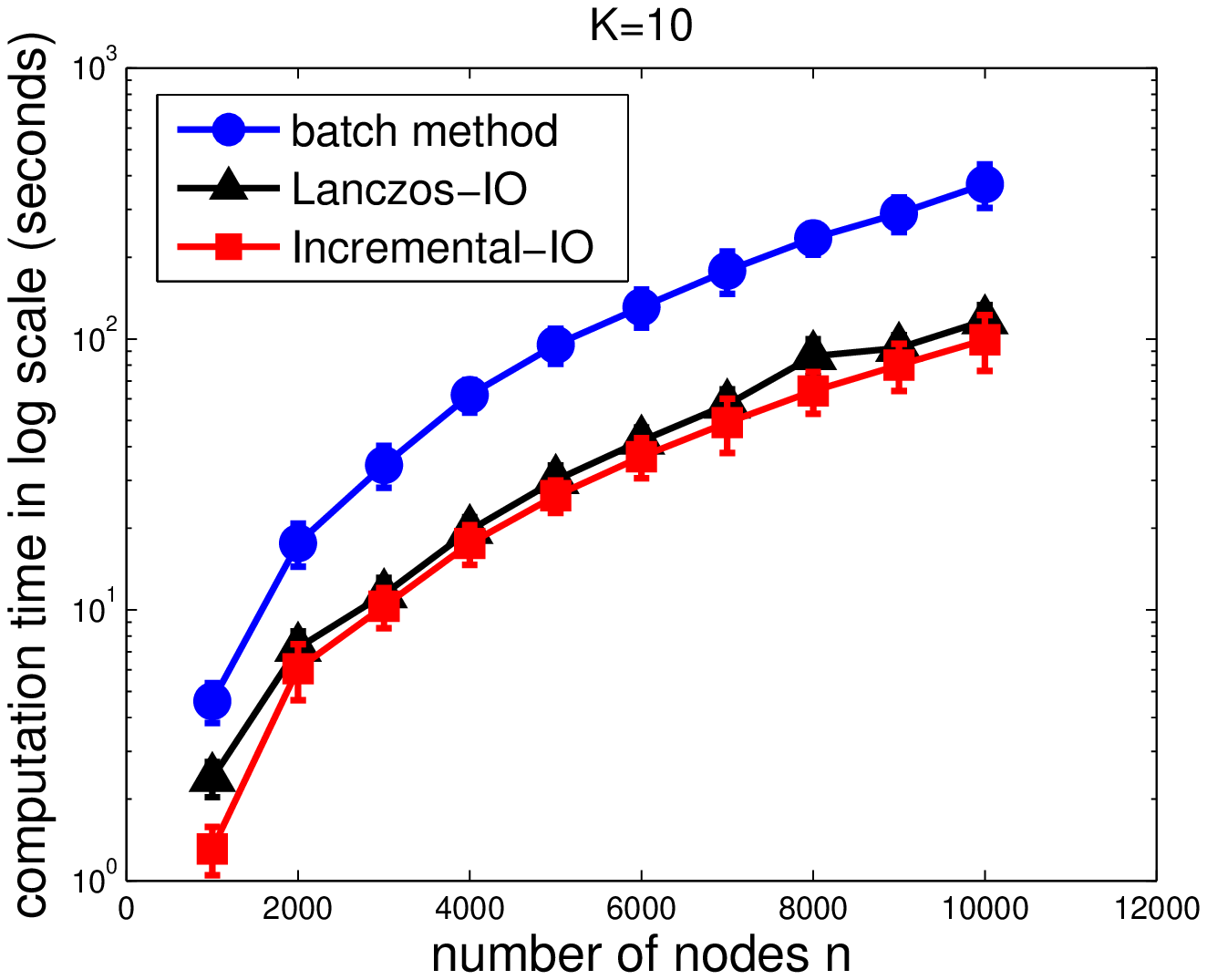}
		\caption{~}
		\label{Fig_incremental_computation_n}
	\end{subfigure}
	\caption{Sequential eigenpair computation time on Erdos-Renyi random graphs with edge connection probability $p=0.1$. The marker represents averaged computation time of 50 trials and the error bar represents standard deviation. 
		(a) Computation time with $n=10000$ and different number of eigenpairs $K$. 
		It is observed that the
			computation time of Incremental-IO and Lanczos-IO grows linearly as $K$
			increases, whereas the computation time of the batch method grows superlinearly
			with $K$. 
		(b) Computation time with $K=10$ and different number of nodes $n$. 
		It is observed that the
			difference in computation time between the batch method and the two incremental
			methods grow polynomially as $n$ increases, which suggests that in this experiment Incremental-IO and Lanczos-IO are more efficient than the
			batch computation method, especially for large graphs.}
	\label{Fig_incremental_computation}
	\vspace{-4mm}
\end{figure*}	



\subsection{Comparison of computation time on real-life datasets}
\label{subsec_time_real}
Fig. \ref{Fig_nor} shows the time improvement of Incremental-IO relative to the batch method for the real-life datasets listed in Table
\ref{tab:statistic}, where the difference in computation time is displayed in
log scale to accommodate large variance of time improvement across datasets
that are of widely varying size.  It is observed that the gain in computational
time via Incremental-IO is more pronounced as cluster count
$K$ increases, which demonstrates the merit of the proposed incremental method. 

\begin{figure}[t]
	\centering
	\includegraphics[width=4in]{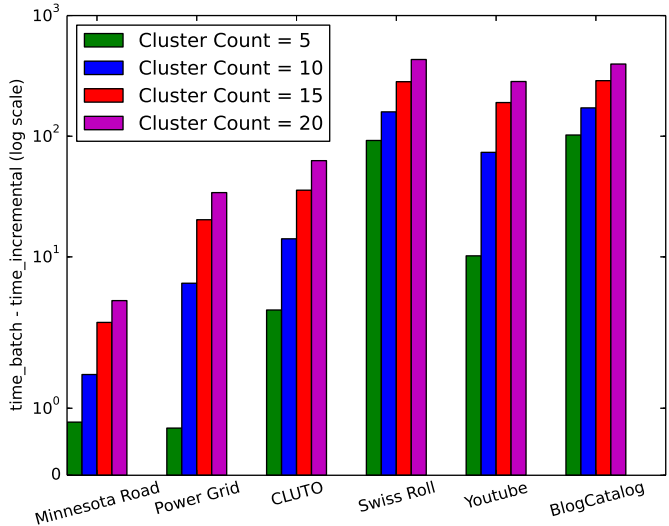}
	\caption{Computation time improvement of Incremental-IO relative to the batch method. Incremental-IO outperforms the batch method for all cases, and has improvement with respect to $K$.}
	\label{Fig_nor}
\end{figure}	

\begin{figure}[t]
	\centering
	\includegraphics[width=4in]{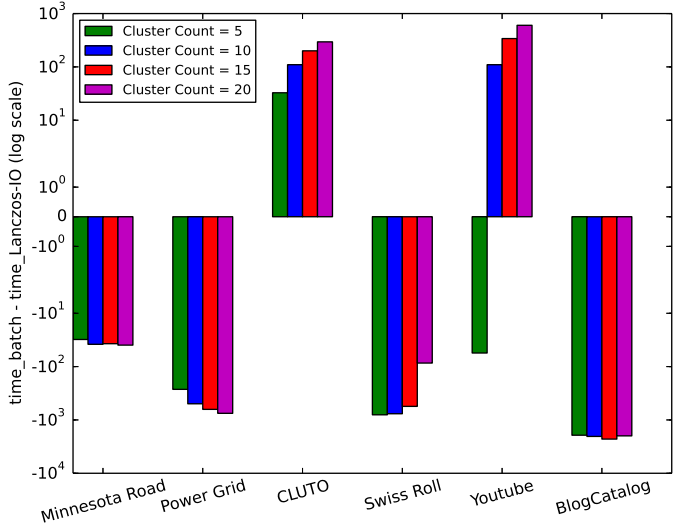}
	\caption{Computation time improvement of Lanczos-IO relative to the batch method. Negative values mean that Lanczos-IO requires more computation time than the batch method. The results suggest that Lanczos-IO is not a robust incremental computation method, as it can perform even worse than the batch method for some cases.}
	\label{Fig_nor2}
\end{figure}

On the other hand, although Lanczos-IO is also an incremental method, in addition to the well-known issue of requiring memory allocation for storing all Lanczos vectors, the experimental results show that it does not provide performance robustness as Incremental-IO does, as it can perform even worse than the batch method for some cases. 	Fig. \ref{Fig_nor2} shows that Lanczos-IO actually results in excessive computation time compared with the batch method for four out of the six datasets, whereas in Fig. \ref{Fig_nor} Incremental-IO is superior than the batch method for all these datasets, which demonstrates the robustness of Incremental-IO over Lanczos-IO.
 The reason of lacking robustness for Lanczos-IO
can be explained by the fact that the previously computed Lanczos vectors may not be effective in minimizing the Ritz approximation error of the desired eigenpairs. In contrast, Incremental-IO and the batch method adopt the implicitly-restarted Lanczos method, which restarts the Lanczos algorithm when the generated Lanczos vectors fail to meet the Ritz approximation criterion, and may eventually lead to faster convergence. Furthermore, Fig. \ref{Fig_effect_Lanczos} shows that Lanczos-IO is overly sensitive to the number of augmented Lanczos vectors $Z_{aug}$, which is a parameter that cannot be optimized \textit{a priori}.

\begin{figure*}[t]
	\centering
	\begin{subfigure}[b]{0.4\linewidth}
		\includegraphics[width=\textwidth]{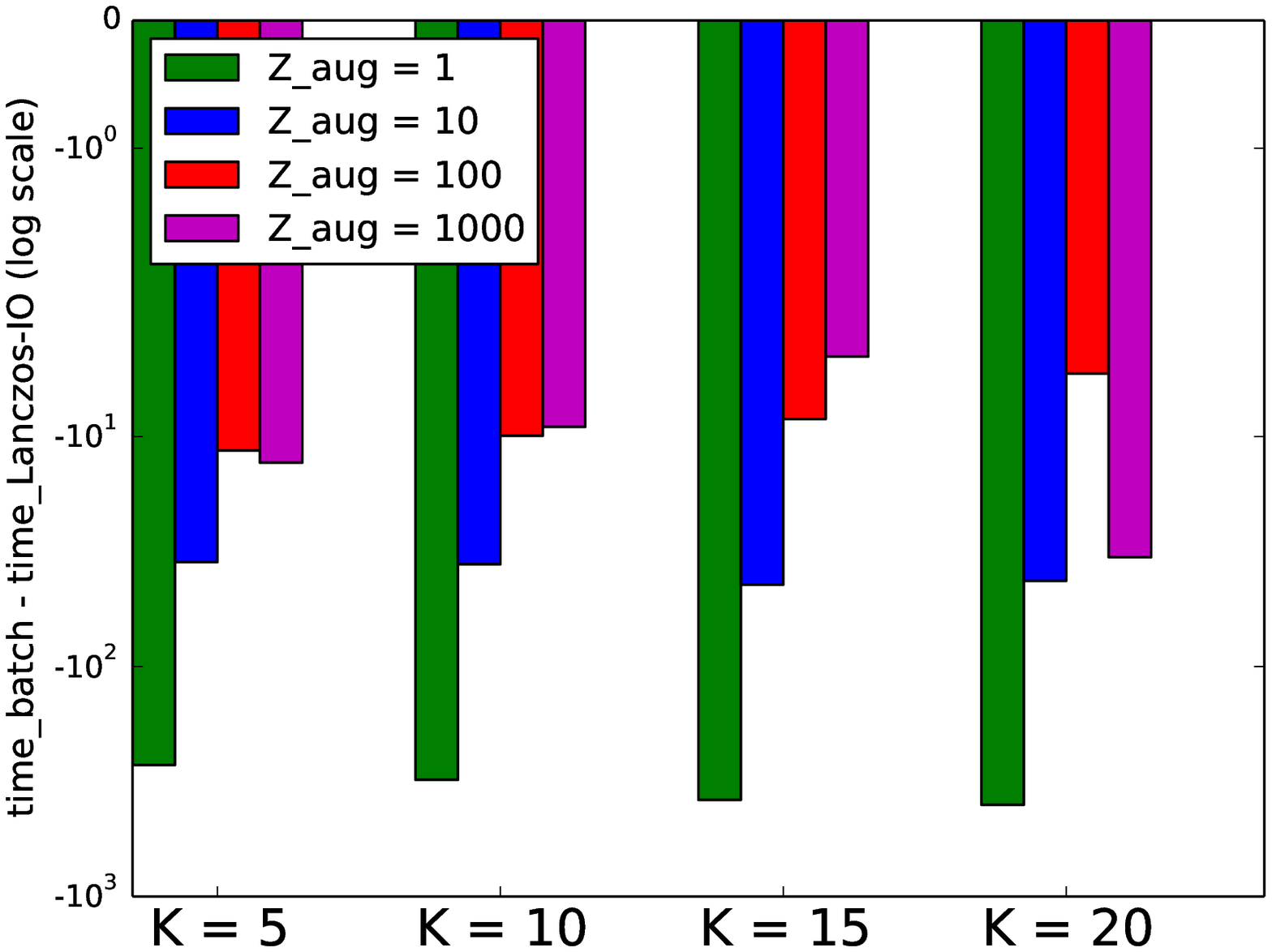}
		\caption{Minnesota Road}
	\end{subfigure}%
	\centering
	\begin{subfigure}[b]{0.4\linewidth}
		\includegraphics[width=\textwidth]{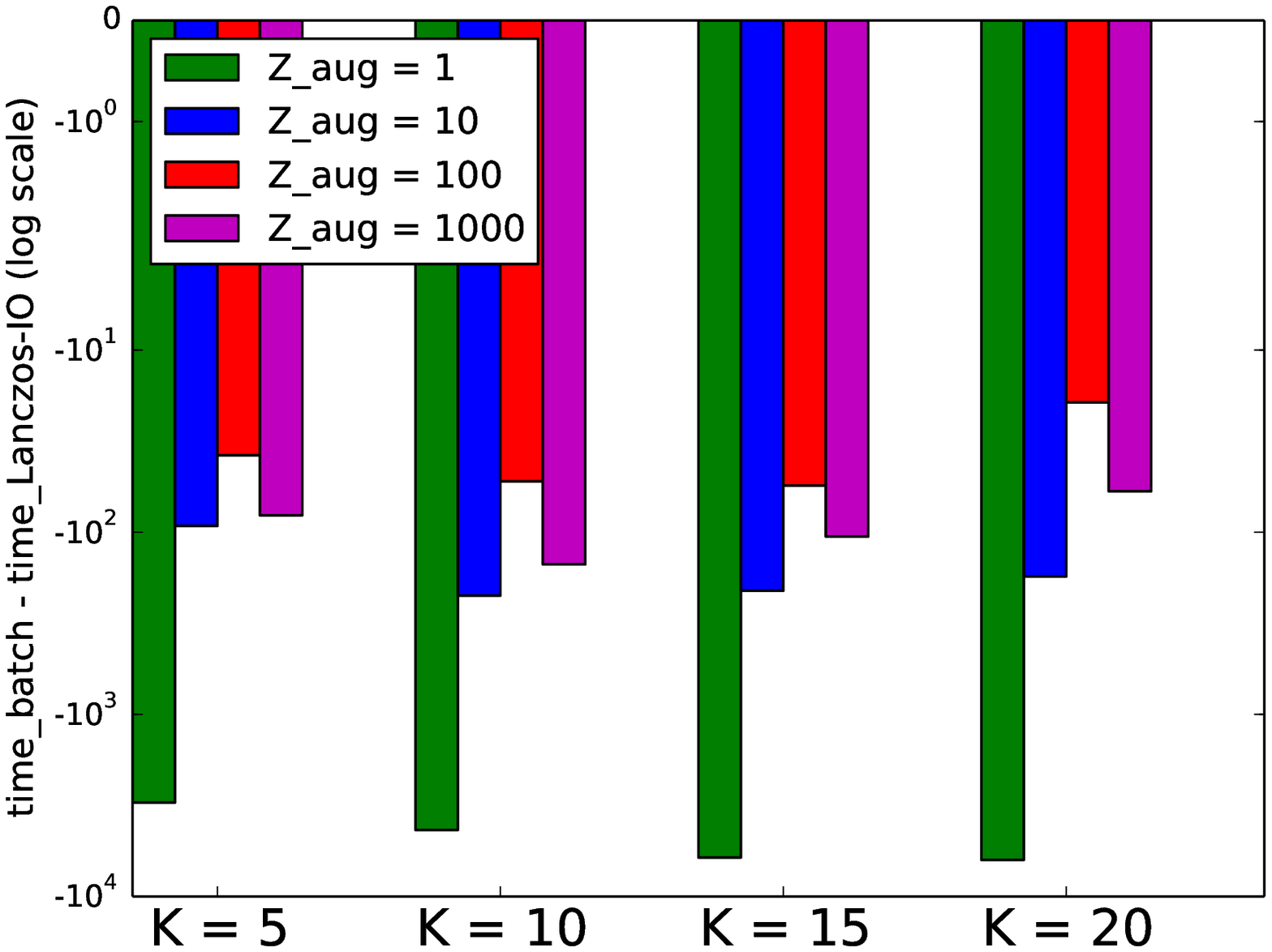}
		\caption{Power Grid}
	\end{subfigure}
	\\
	\centering
	\begin{subfigure}[b]{0.4\linewidth}
		\includegraphics[width=\textwidth]{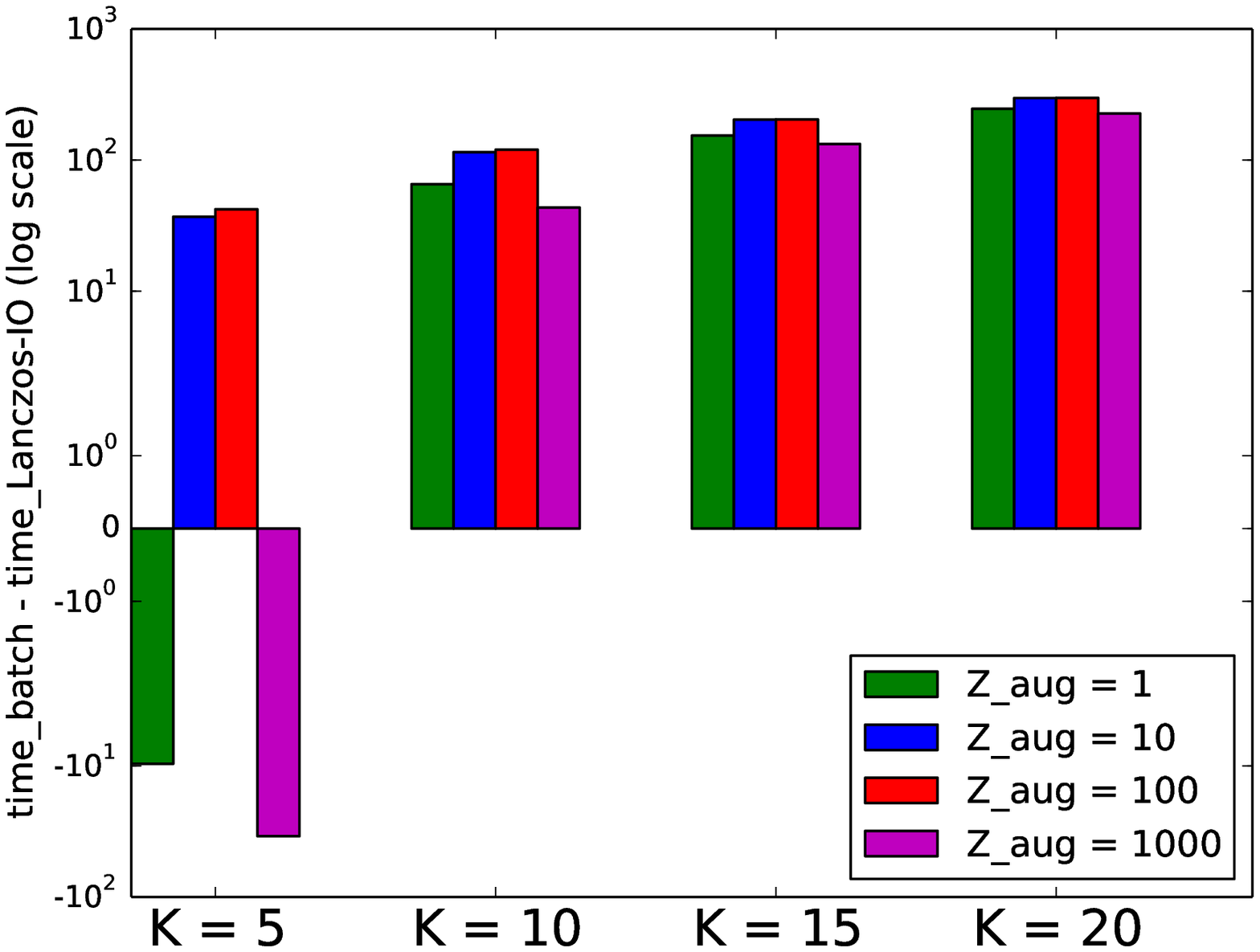}
		\caption{CLUTO}
	\end{subfigure}	
	\centering
	\begin{subfigure}[b]{0.4\linewidth}
		\includegraphics[width=\textwidth]{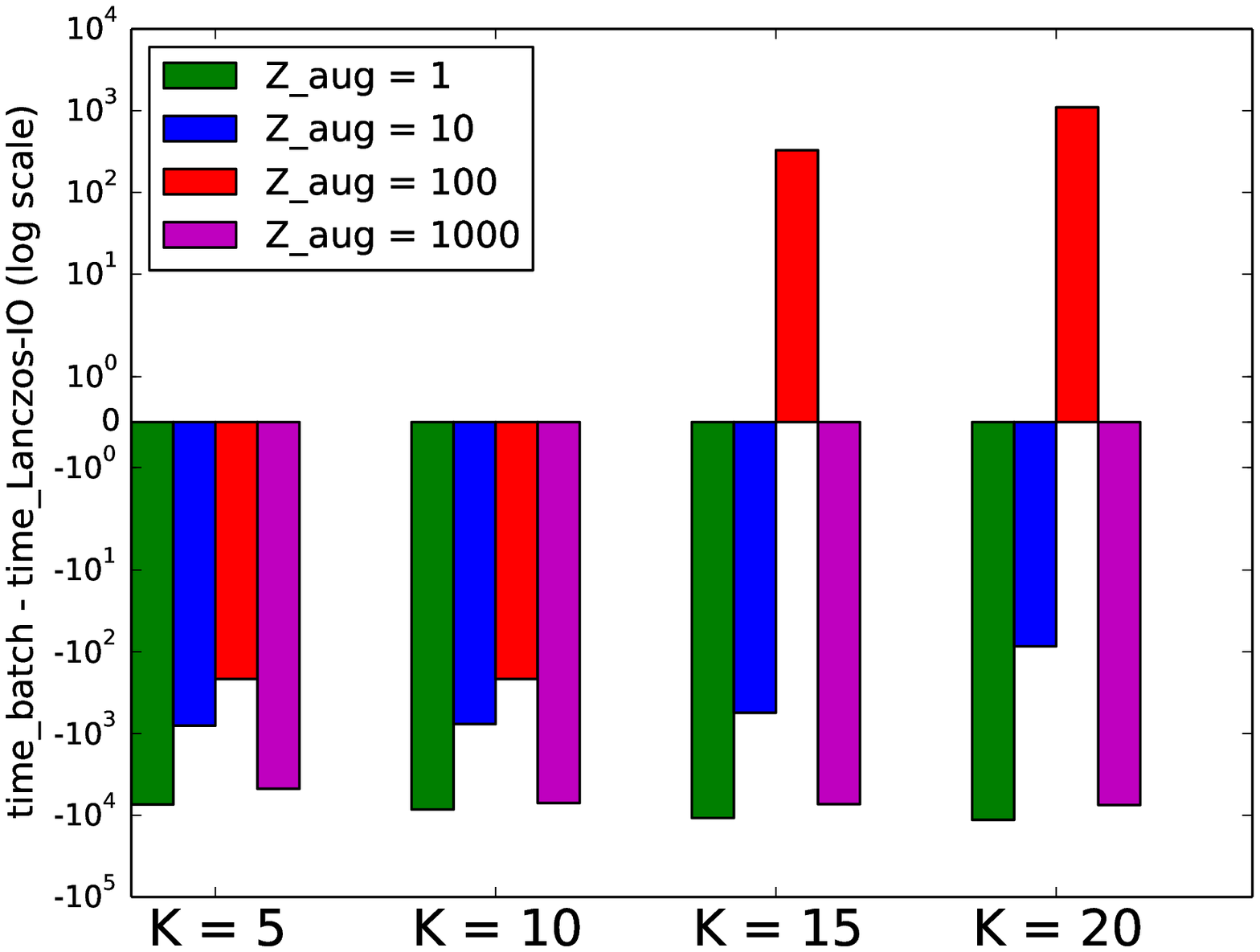}
		\caption{Swiss Roll}
	\end{subfigure}		
	\\
	\centering
	\begin{subfigure}[b]{0.4\linewidth}
		\includegraphics[width=\textwidth]{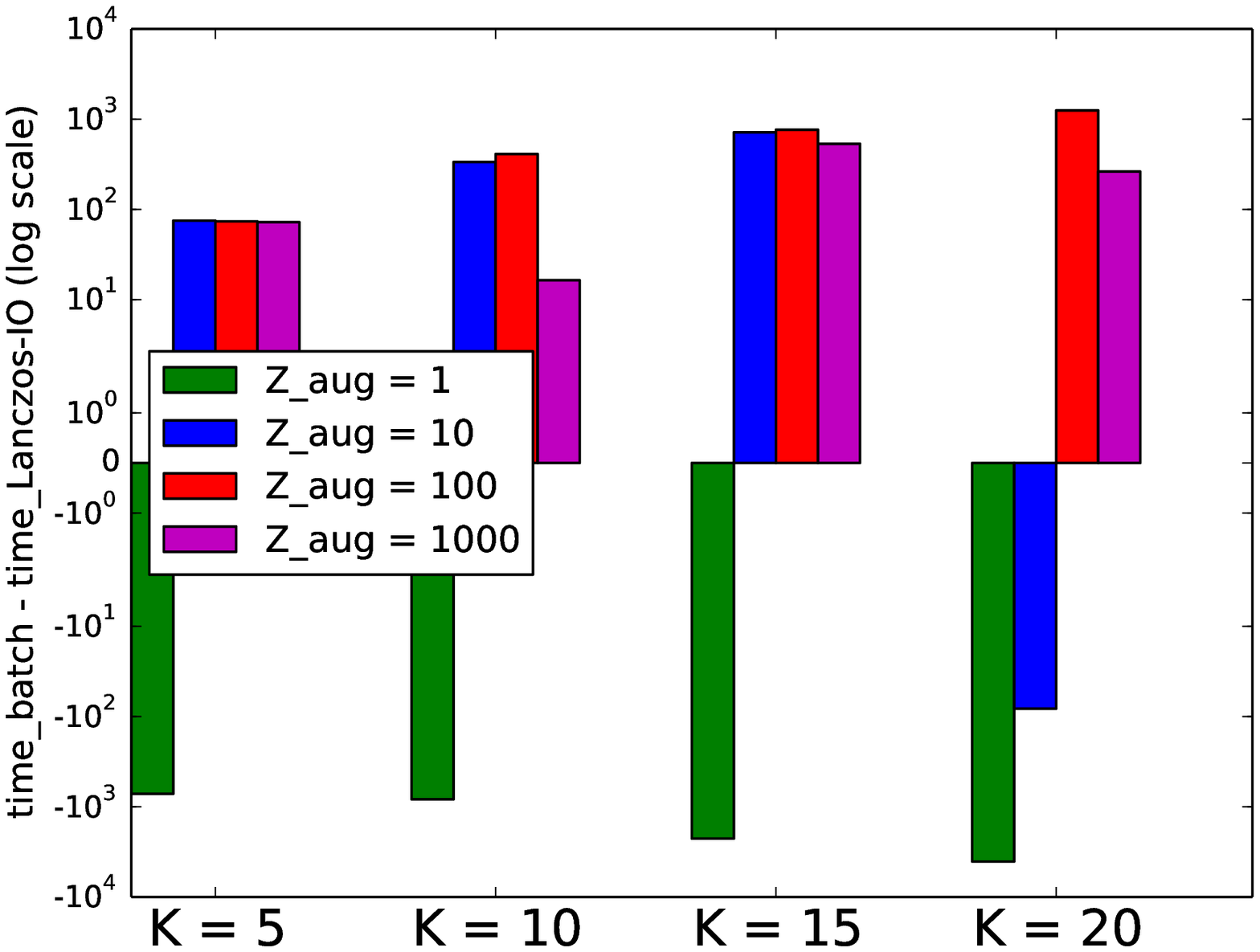}
		\caption{Youtube}
	\end{subfigure}
	\centering
	\begin{subfigure}[b]{0.4\linewidth}
		\includegraphics[width=\textwidth]{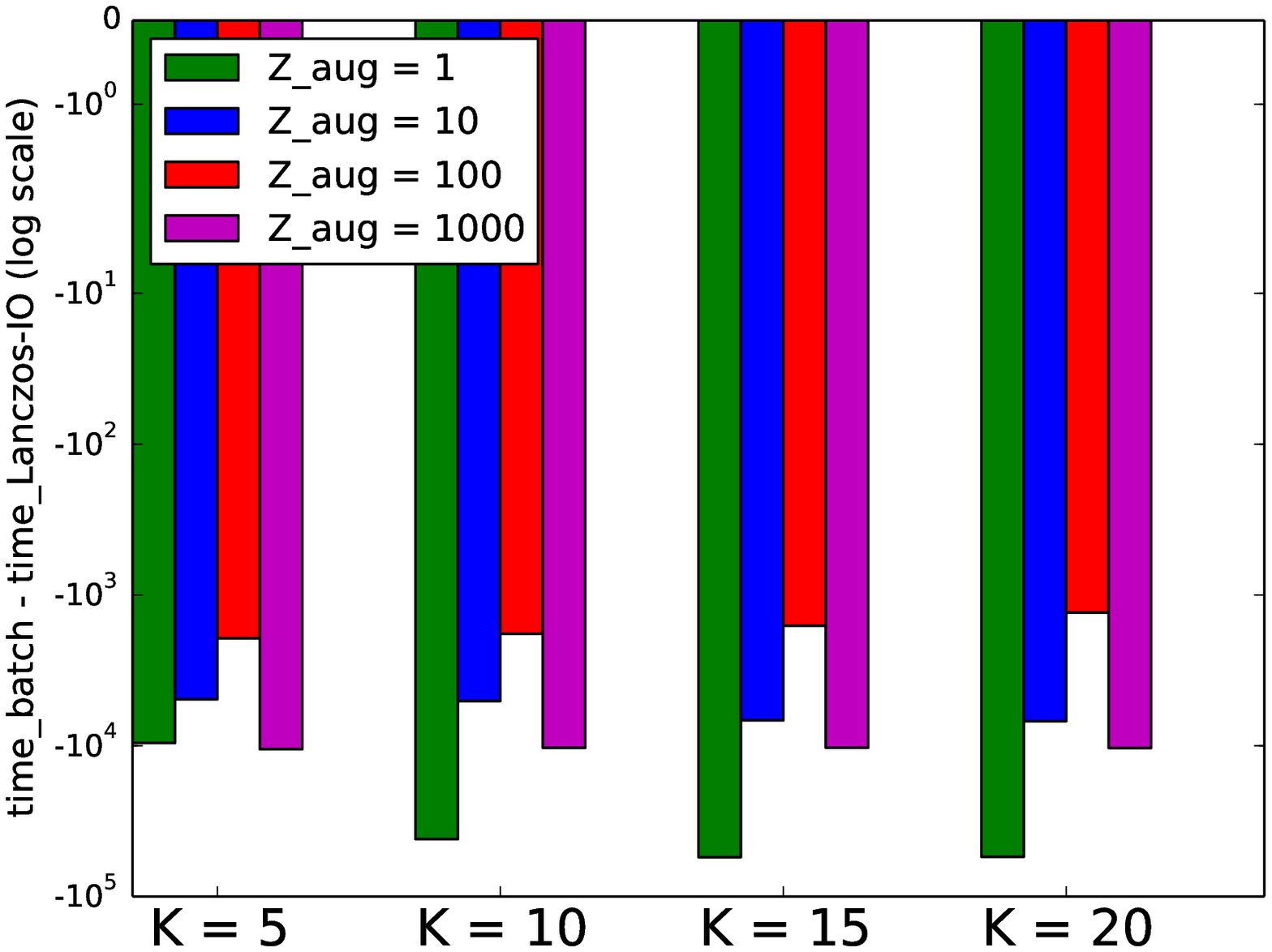}
		\caption{BlogCatalog}
	\end{subfigure}									
	\caption{The effect of number of augmented Lanczos vectors $Z_{aug}$ of Lanczos-IO in \textbf{Algorithm} \ref{algo_LanczosIO} on computation time improvement relative to the batch method. Negative values mean that the computation time of Lanczos-IO is larger than that of the batch method. The results show that Lanczos-IO is not a robust incremental eigenpair computation method. Intuitively, small $Z_{aug}$ may incur many iterations in the second step of \textbf{Algorithm} \ref{algo_LanczosIO}, whereas large $Z_{aug}$ may pose computation burden in the first step of \textbf{Algorithm} \ref{algo_LanczosIO}, and therefore both cases lead to the increase in computation time.}
	\label{Fig_effect_Lanczos}
\end{figure*}


\textbf{Theorem}~\ref{thm_connect_GL} establishes that the proposed incremental method (Incremental-IO)
exactly computes the $K$-th eigenpair using 1 to $(K-1)$-th eigenpairs, yet for
the sake of experiments with real datasets, we have computed the normed
eigenvalue difference (in terms of root mean squared error) and the correlations of the $K$ smallest eigenvectors obtained from the batch method and Incremental-IO. As displayed in Fig. \ref{Fig_RMSE_eigvalue}, the $K$ smallest eigenpairs are identical as expected; to be more specific, using
Matlab library, on the Minnesota road dataset for $K=20$, the normed
eigenvalue difference is $7\times 10^{-12}$ and the associated eigenvectors are identical up to differences in sign. 
For all datasets listed in Table \ref{tab:statistic},  the normed eigenvalue difference is negligible and the associated eigenvectors are identical up to the difference in sign, i.e., the eigenvector correlation in magnitude equals to 1 for every pair of corresponding eigenvectors of the two methods, which verifies the correctness of  Incremental-IO. 
Moreover, due to the eigenpair consistency between the batch method and Incremental-IO as demonstrated in  Fig. \ref{Fig_RMSE_eigvalue}, they yield the same clustering results in the considered datasets.

\begin{figure*}[]
	\centering
	\begin{subfigure}[b]{0.4\textwidth}
		\includegraphics[width=\textwidth]{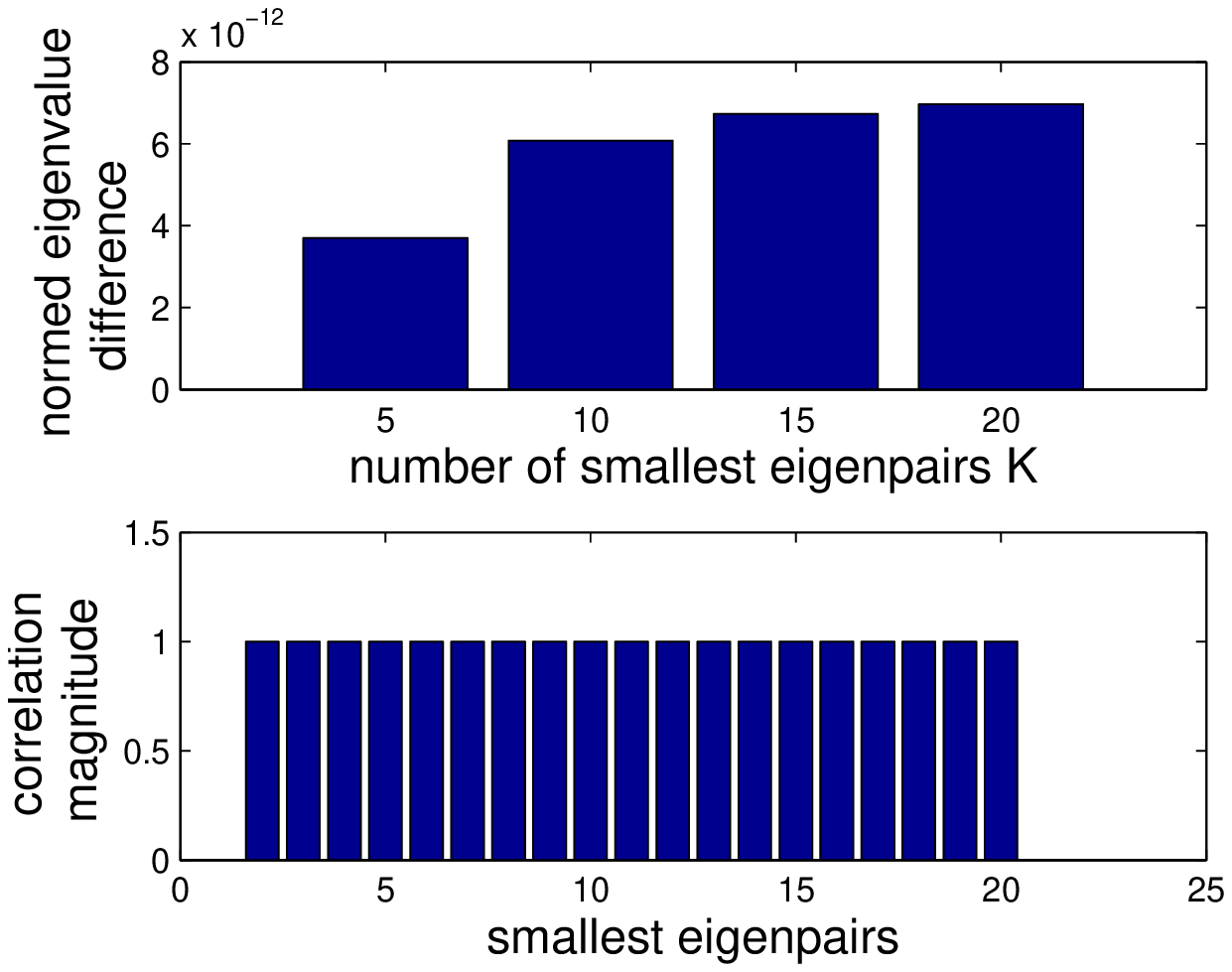}
		\caption{Minnesota Road}
		\label{Fig_RMSE_eigvalue_Minnesota}
	\end{subfigure}%
	\centering
	\begin{subfigure}[b]{0.4\textwidth}
		\includegraphics[width=\textwidth]{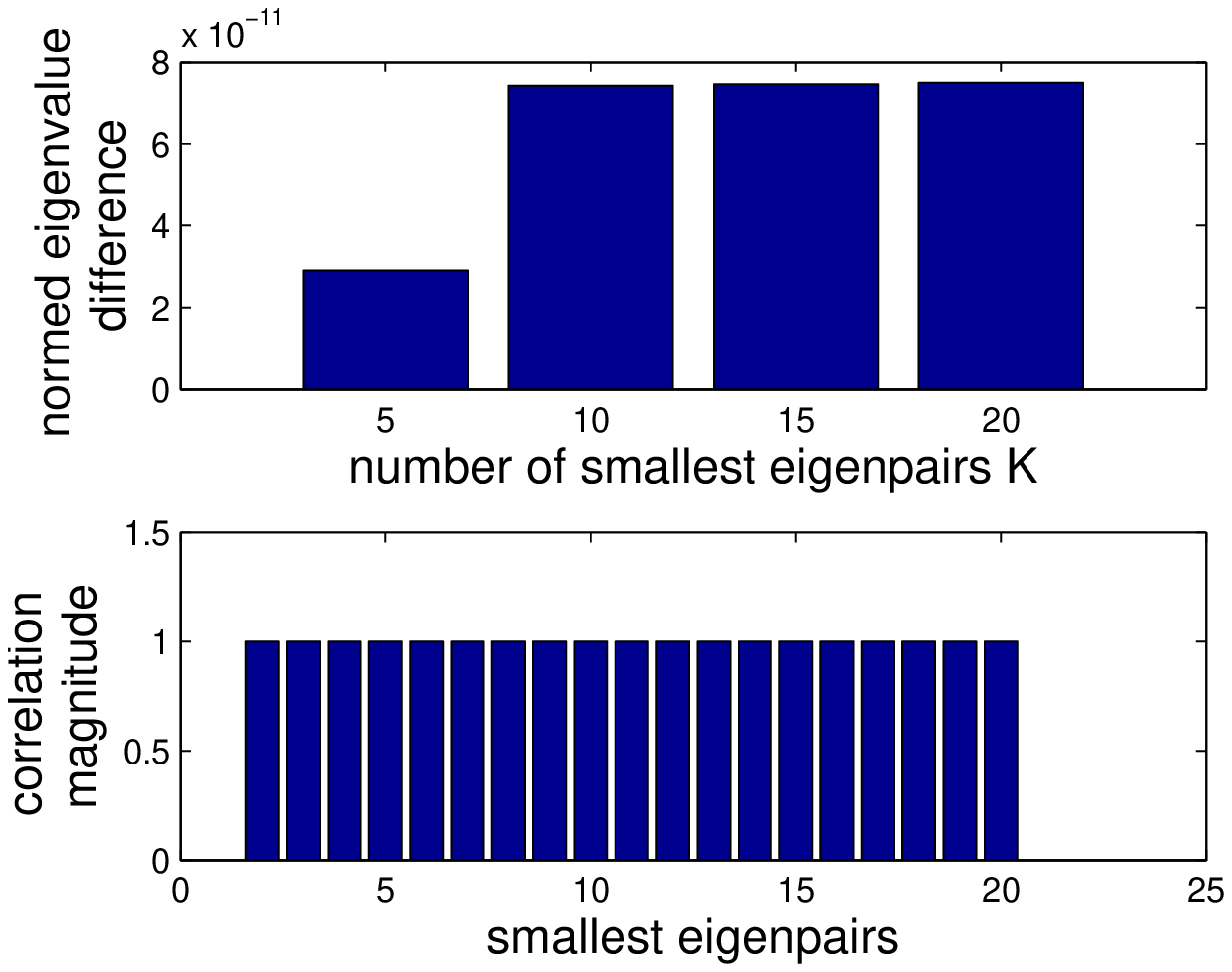}
		\caption{Power Grid}
		\label{Fig_RMSE_eigvalue_Power}
	\end{subfigure}%
	\\
	\centering
	\begin{subfigure}[b]{0.4\textwidth}
		\includegraphics[width=\textwidth]{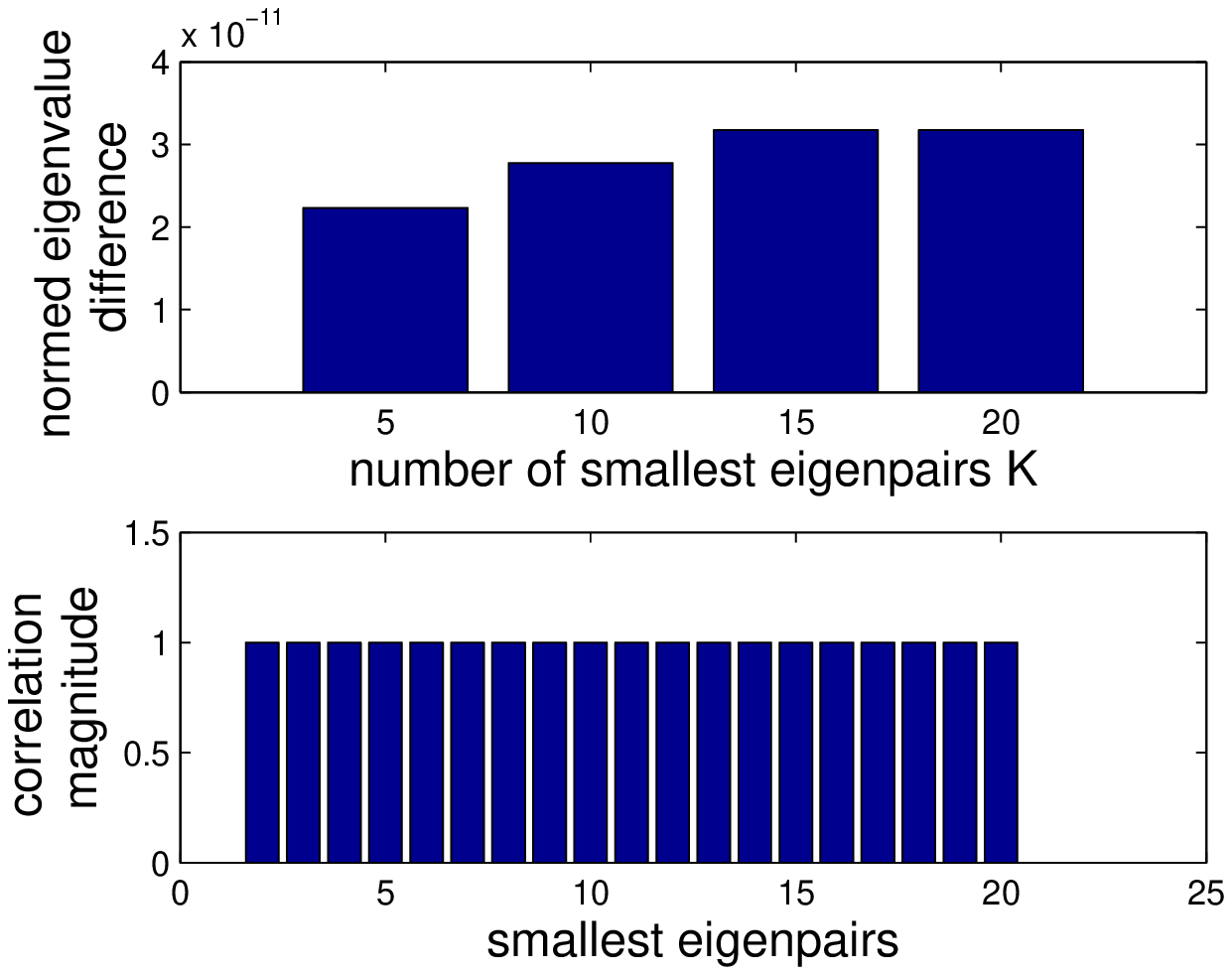}
		\caption{CLUTO}
		\label{Fig_RMSE_eigvalue_CLUTO}
	\end{subfigure}
	\centering
	\begin{subfigure}[b]{0.4\textwidth}
		\includegraphics[width=\textwidth]{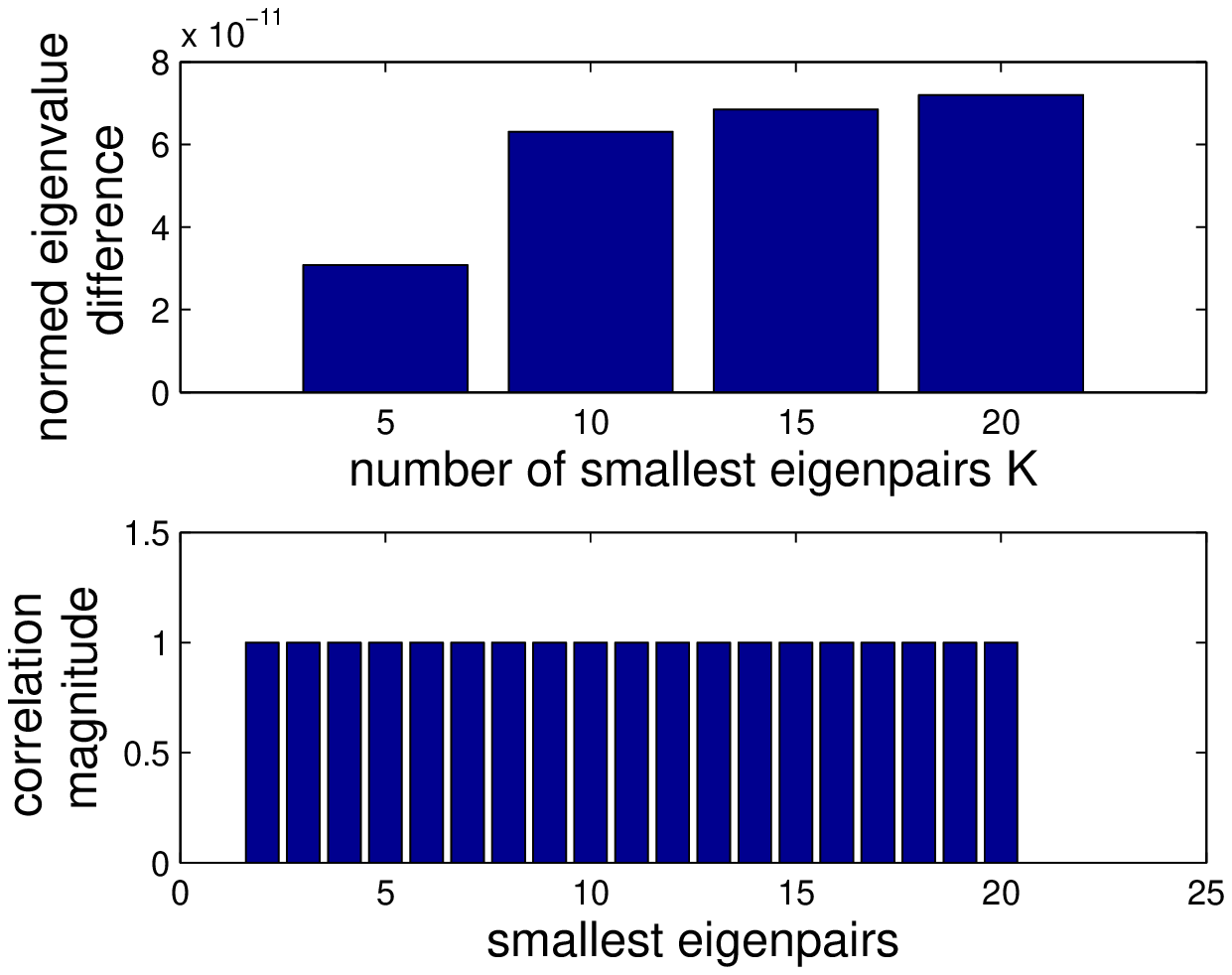}
		\caption{Swiss Roll}
		\label{Fig_RMSE_eigvalue_Swiss}
	\end{subfigure}
	\\
	\centering
	\begin{subfigure}[b]{0.4\textwidth}
		\includegraphics[width=\textwidth]{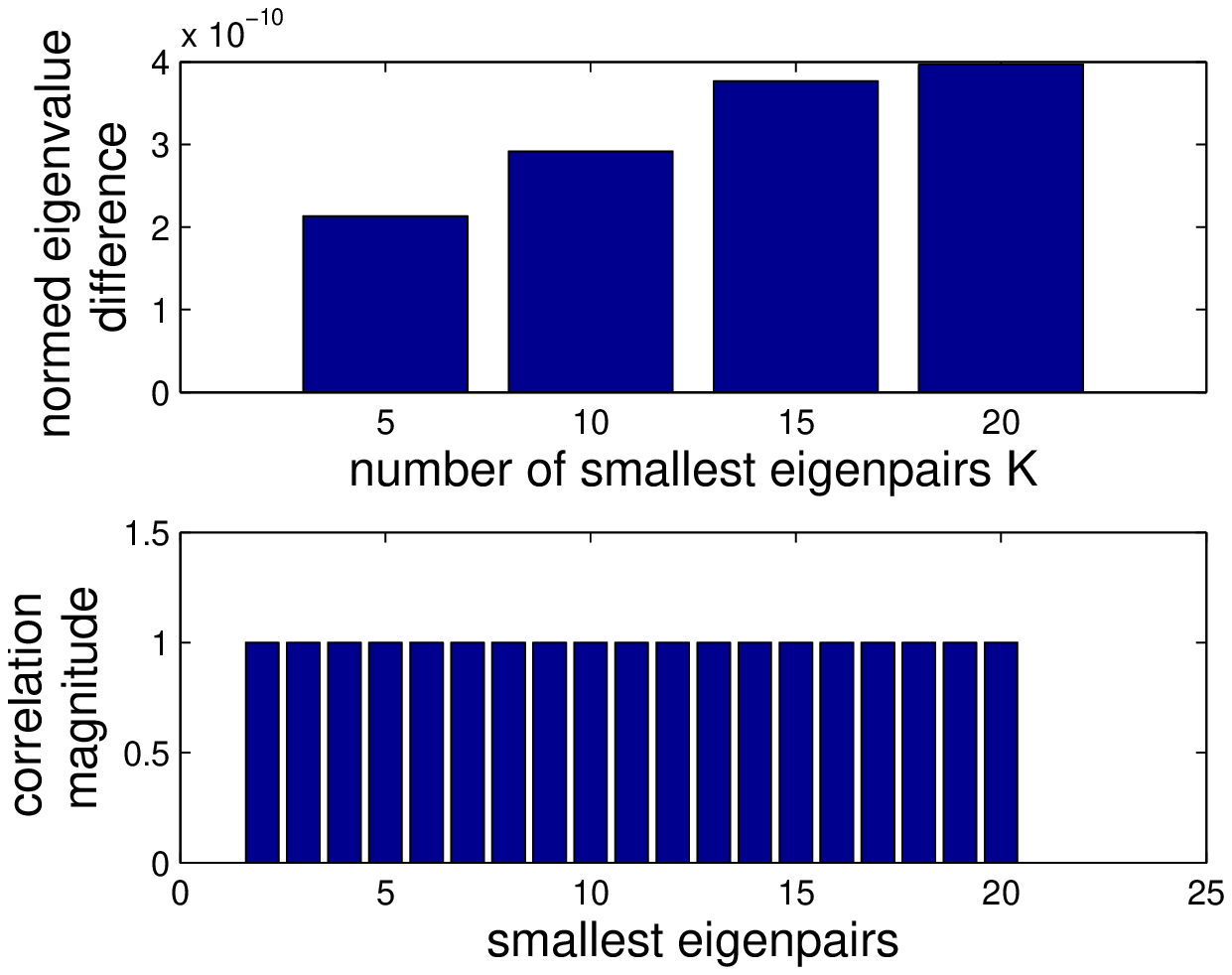}
		\caption{Youtube}
		\label{Fig_RMSE_eigvalue_Youtube}
	\end{subfigure}%
	\centering
	\begin{subfigure}[b]{0.4\textwidth}
		\includegraphics[width=\textwidth]{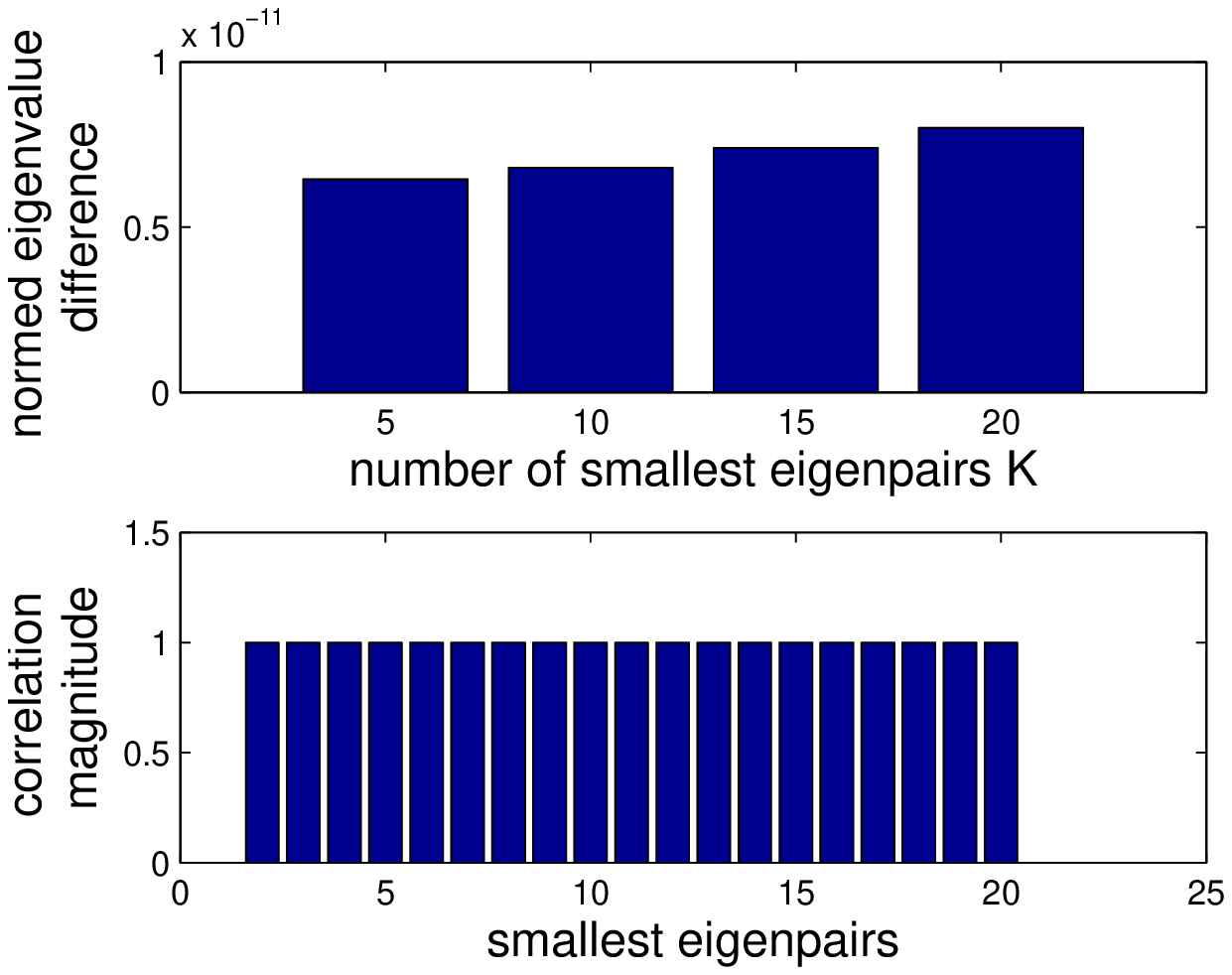}
		\caption{BlogCatalog}
		\label{Fig_RMSE_eigvalue_Blog}
	\end{subfigure}%
	\caption{Consistency of smallest eigenpairs computed by the batch computation method and Incremental-IO for datasets listed in Table \ref{tab:statistic}.
		The normed eigenvalue difference is the square root of sum of squared differences between eigenvalues. The correlation magnitude is the absolute value of inner product of eigenvectors, where 1 means perfect alignment.}
	\label{Fig_RMSE_eigvalue}							
\end{figure*}

\subsection{Clustering metrics for user-guided spectral clustering}
\label{subsec_clustering_metric}
In real-life, an analyst can use Incremental-IO for clustering along with a
mechanism for selecting the best choice of $K$ starting from $K=2$. To
demonstrate this, in the experiment we use five clustering metrics that can be
used for online decision making regarding the value of $K$. These metrics are
commonly used in clustering unweighted and weighted graphs and they are summarized as
follows.

\noindent \textbf{1. Modularity:} modularity is defined as
\begin{align}
\text{Mod} = \displaystyle \sum_{i=1}^{K} \bigg(\frac{W(\mathcal{C}_i,\mathcal{C}_i)}{W(\mathcal{V},\mathcal{V})} - \bigg(\frac{W(\mathcal{C}_i,\mathcal{V})}{W(\mathcal{V},\mathcal{V})}\bigg)^{2}\bigg),
\end{align}
where $\mathcal{V}$ is the set of all nodes in the graph, $\mathcal{C}_i$ is the $i$-th cluster,  $W(\mathcal{C}_i,\mathcal{C}_i)$ ($W(\mathcal{C}_i,\overline{\mathcal{C}_i})$) denotes the sum of weights of all internal (external) edges of the $i$-th cluster, 	$W(\mathcal{C}_i,\mathcal{V}) = \displaystyle W(\mathcal{C}_i, \mathcal{C}_i) + W(\mathcal{C}_i,\overline{\mathcal{C}_i})$, and
$W(\mathcal{V},\mathcal{V}) =  \sum_{j=1}^{n} s_j=s$ denotes the total nodal strength.

\noindent \textbf{2. Scaled normalized cut (SNC):}  NC is defined as \cite{Zaki.Jr:14} 
\begin{align}
\label{eq:NC}
&\text{NC} = \displaystyle \sum_{i=1}^{K} \frac{ W(\mathcal{C}_i, \overline{\mathcal{C}_i})}{W(\mathcal{C}_i, \mathcal{V})}.  
\end{align}
SNC is NC divided by the number of clusters, i.e., NC/$K$.

\noindent \textbf{3. Scaled median (or maximum) cluster size:} Scaled medium (maximum) cluster size is the medium (maximum) cluster size of $K$ clusters divided by the total number of nodes $n$ of a graph.

\noindent \textbf{4. Scaled spectrum energy:} scaled spectrum energy is
the sum of the $K$ smallest eigenvalues of the graph Laplacian matrix $\bL$
divided by the sum of all eigenvalues of $\bL$, which can be 
computed by
\begin{align}
\text{scaled spectrum energy}=\frac{\sum_{i=1}^{K} \lambda_i(\bL)}{\sum_{j=1}^n \bL_{jj}},		
\end{align}
where $\lambda_i(\bL)$ is the $i$-th smallest eigenvalue of $\bL$ and $\sum_{j=1}^n \bL_{jj}$ 
$=\sum_{i=1}^{n} \lambda_i(\bL)$ is the sum of diagonal elements of $\bL$.

These metrics provide alternatives for gauging the quality of the clustering method. For example, Mod and NC reflect the trade-off between intracluster
similarity and intercluster separation. Therefore, the larger the value of Mod, the
better the clustering quality, and the smaller the value of NC, the better
the clustering quality. Scaled spectrum energy is a typical measure of cluster
quality for spectral clustering
\cite{Polito01grouping,ng2002spectral,zelnik2004self}, and smaller spectrum
energy means better separability of clusters. 
For Mod and scaled NC, a user might look for a cluster count $K$ such that the increment in the clustering metric is not significant, i.e., the clustering metric is saturated beyond such a $K$. For scaled median and maximum cluster size, 
a user might require the cluster count $K$ to be such that the clustering metric is below a desired value. For scaled spectrum energy, a user might look for a noticeable increase in the clustering metric between consecutive values of $K$.

\subsection{Demonstration}
Here we use Minnesota Road data to demonstrate how users can utilize the clustering metrics in Sec. \ref{subsec_clustering_metric} to determine the number of clusters.
For example, the five metrics evaluated for Minnesota Road clustering with respect to different cluster counts $K$ are displayed in Fig. \ref{Fig_clustering_metric} (a).
Starting from $K=2$ clusters, these metrics are updated by the incremental user-guided spectral clustering algorithm
(\textbf{Algorithm}~\ref{algo_incremental_automated_clustering}) as $K$ increases.
If the user imposes that the maximum cluster size should be less than $30\%$ of the total number of nodes, then the algorithm 
returns clustering results with a number of clusters of $K=6$ or greater.
Inspecting the modularity one sees it saturates at $K=7$, and the user also observes a noticeable increase in scaled spectrum energy when $K=7$. Therefore, the algorithm can be used to incrementally generate four clustering results for $K=7,8,9$, and $10$.
The selected clustering results in Fig. \ref{Fig_Minnesota_cluster_visualization}  are shown to be consistent with geographic separations of different granularity.

\begin{figure*}[]
	\centering
	\begin{subfigure}[b]{0.4\textwidth}
		\includegraphics[width=\textwidth]{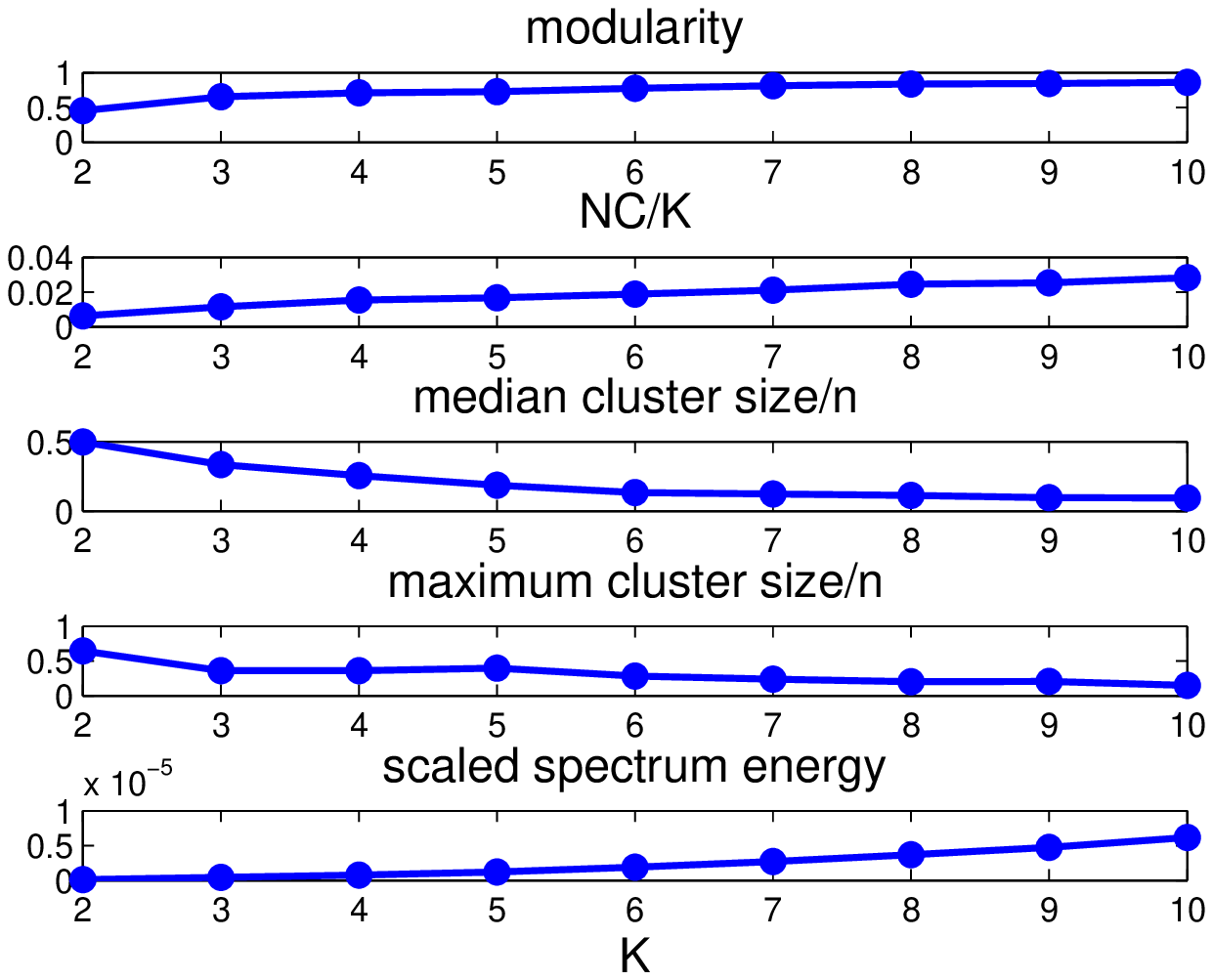}
		\caption{Minnesota Road}
		\label{Fig_Minnesota_Road_metric2}
	\end{subfigure}%
	\centering
	\begin{subfigure}[b]{0.4\textwidth}
		\includegraphics[width=\textwidth]{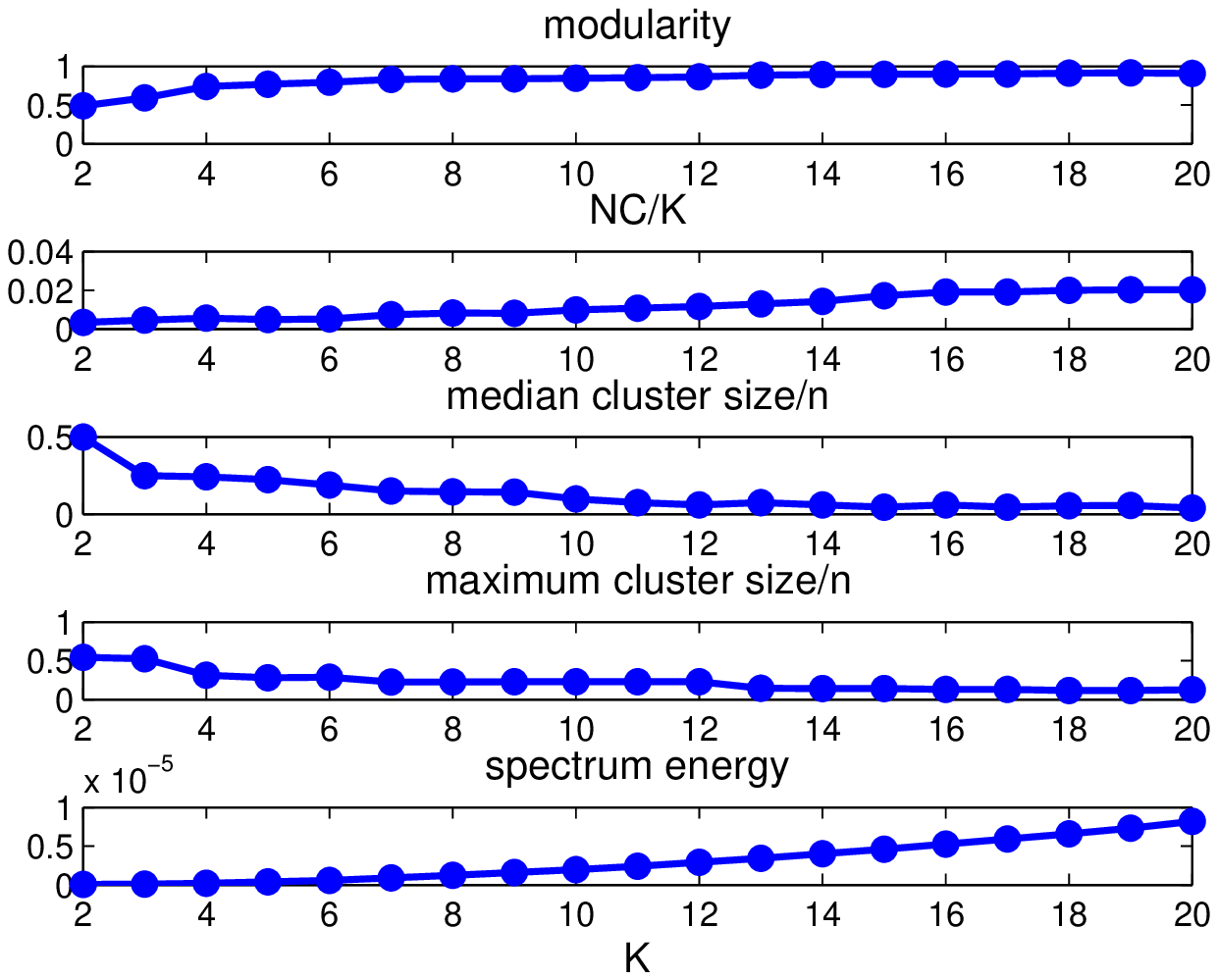}
		\caption{Power Grid}
		\label{Fig_Newman_Power_metric}
	\end{subfigure}%
	\\
	\centering
	\begin{subfigure}[b]{0.4\textwidth}
		\includegraphics[width=\textwidth]{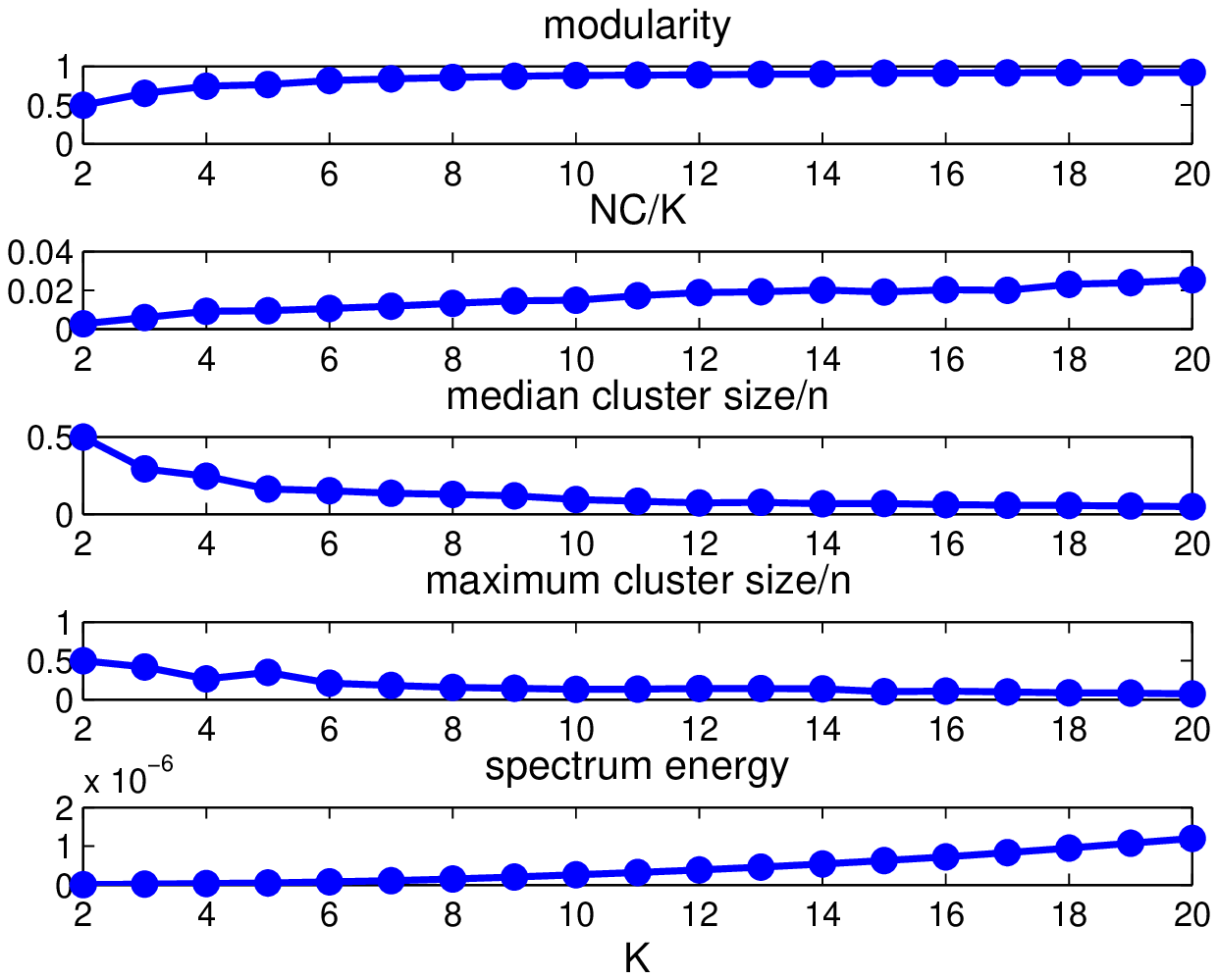}
		\caption{Swiss Roll}
		\label{Fig_Swiss_metric}
	\end{subfigure}%
	\centering
	\begin{subfigure}[b]{0.4\textwidth}
		\includegraphics[width=\textwidth]{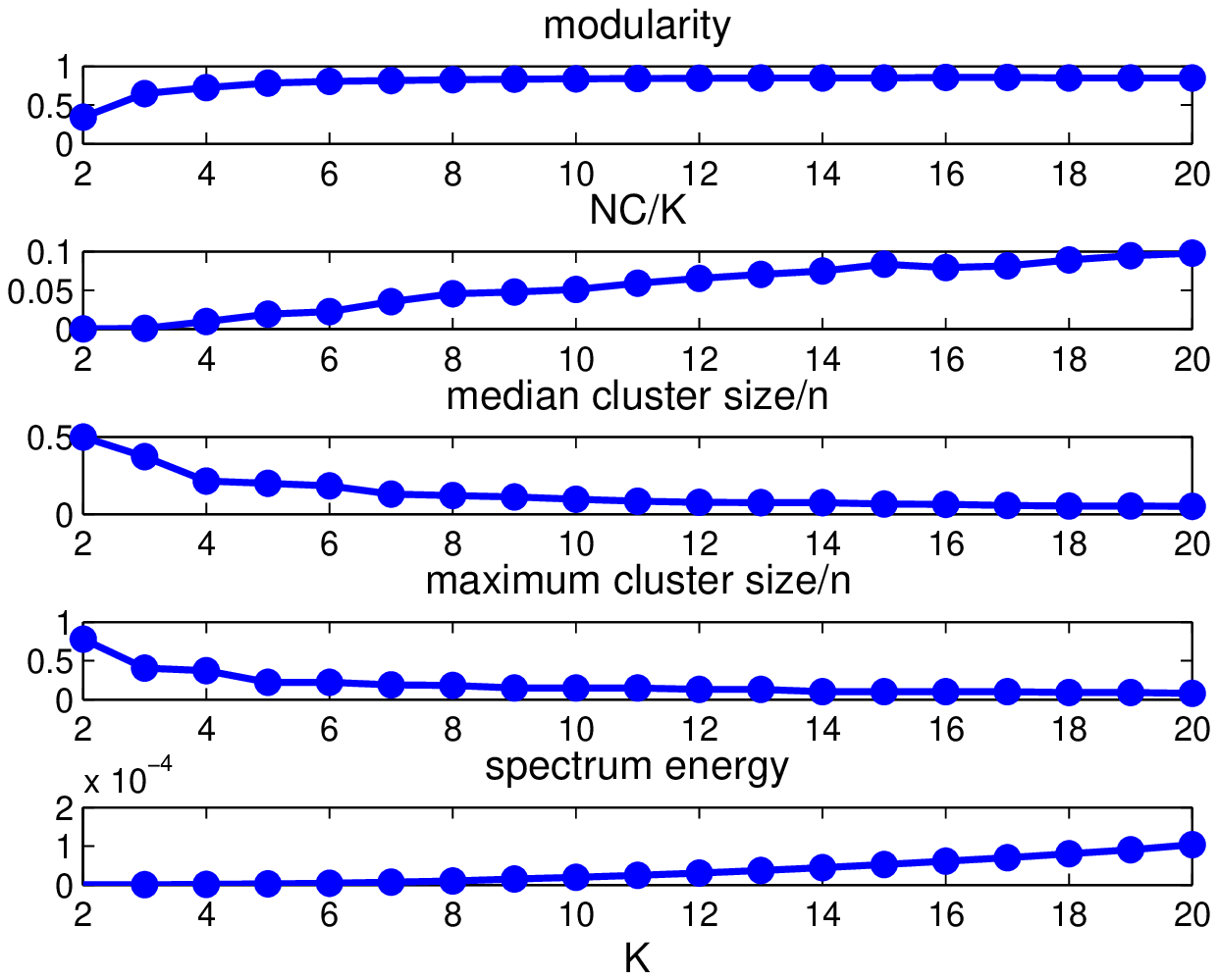}
		\caption{CLUTO}
		\label{Fig_cluto_t7_clean}			
	\end{subfigure}%
	\\
	\centering
	\begin{subfigure}[b]{0.4\textwidth}
		\includegraphics[width=\textwidth]{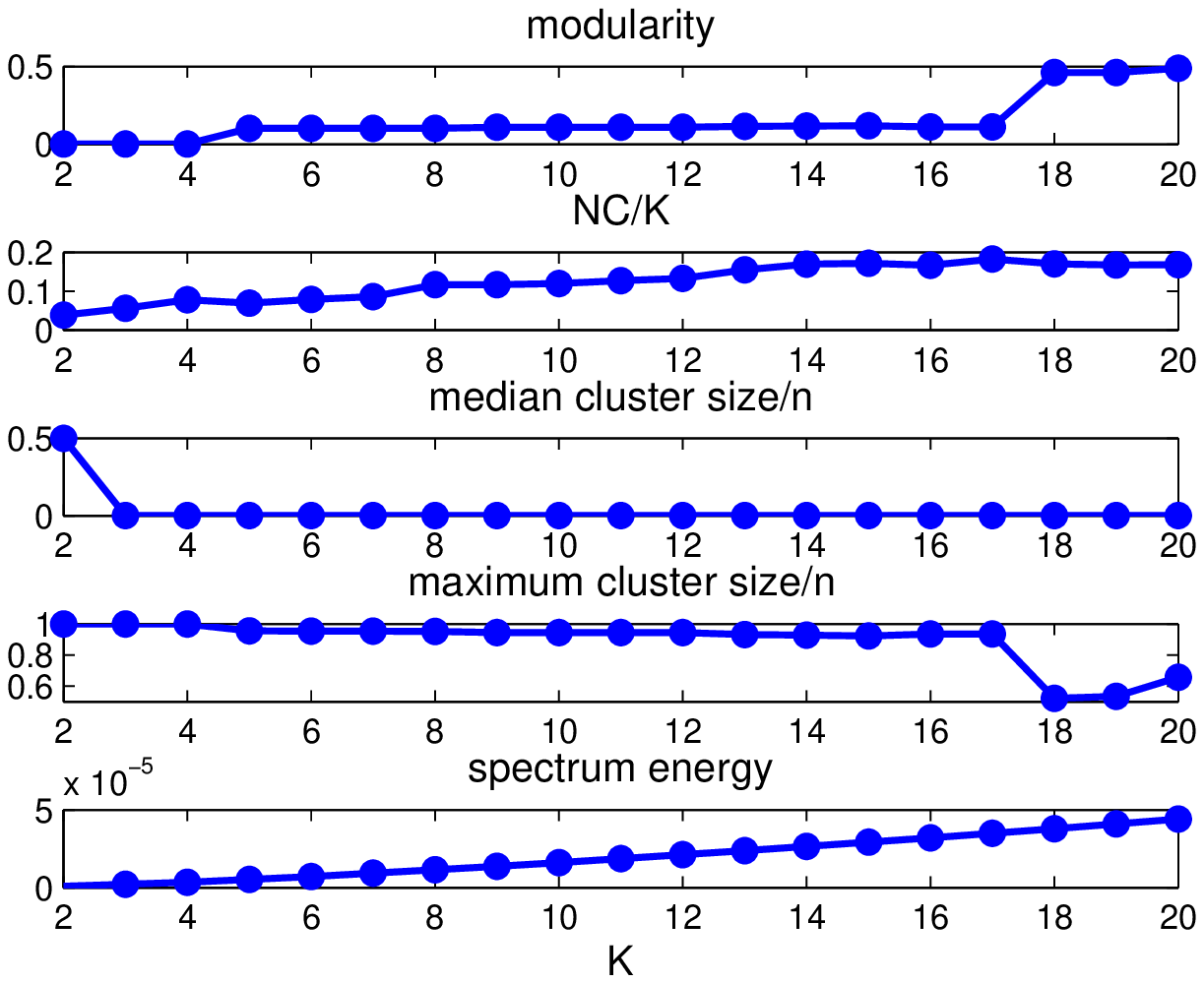}
		\caption{Youtube}
		\label{Fig_youtube_1}
	\end{subfigure}%
	\centering
	\begin{subfigure}[b]{0.4\textwidth}
		\includegraphics[width=\textwidth]{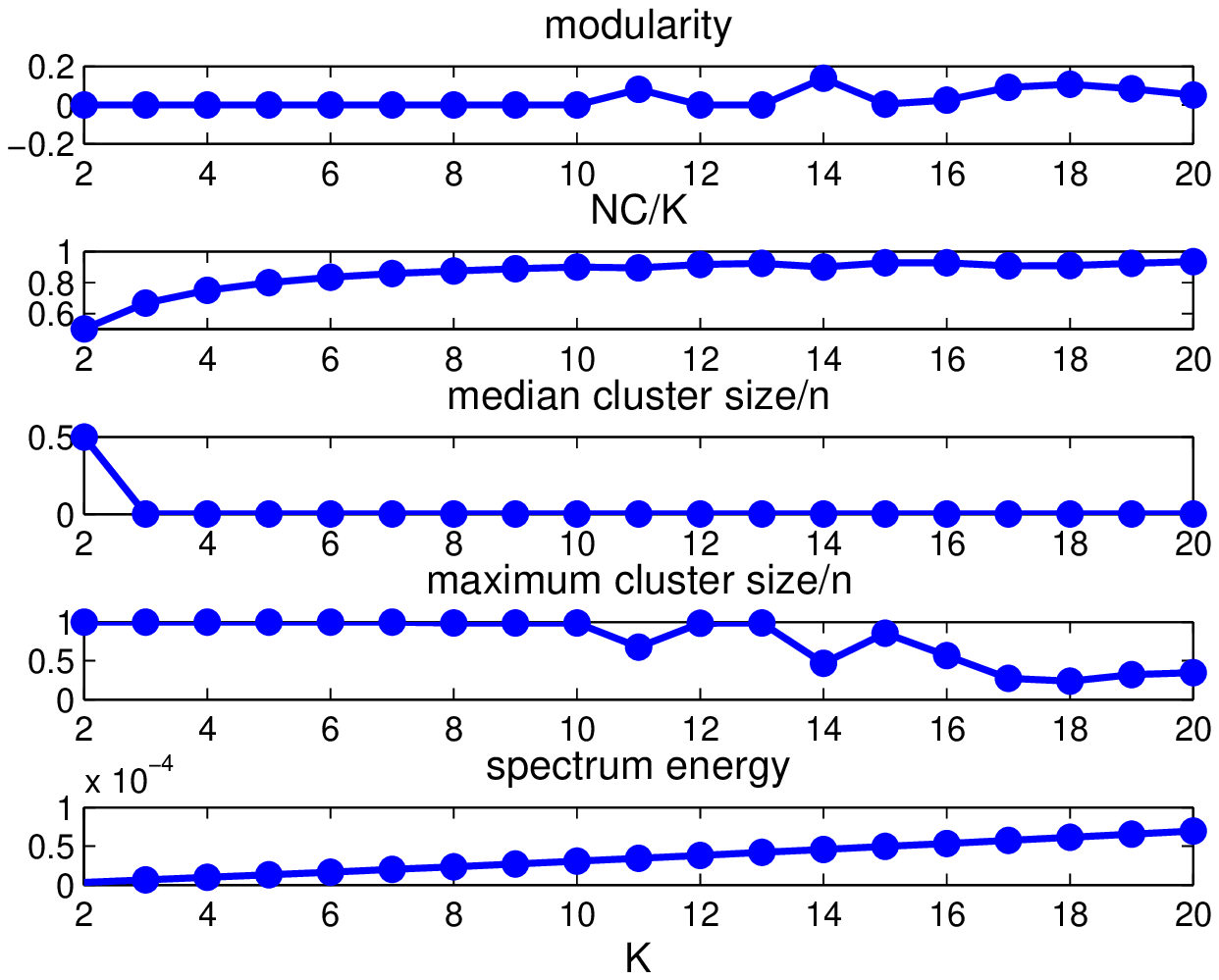}
		\caption{BlogCatalog}
		\label{Fig_blog}
	\end{subfigure}%
	\caption{Five clustering metrics computed incrementally via \textbf{Algorithm} \ref{algo_incremental_automated_clustering}
		for different datasets listed in Table \ref{tab:statistic}. The metrics are modularity, scaled normalized cut (NC/$K$), scaled median and maximum cluster size, and scaled spectrum energy. These clustering metrics are used to help users determine the number of clusters.}
	\label{Fig_clustering_metric}						
\end{figure*}

\begin{figure*}[]
	\centering
	\begin{subfigure}[b]{0.4\textwidth}
		\includegraphics[width=\textwidth]{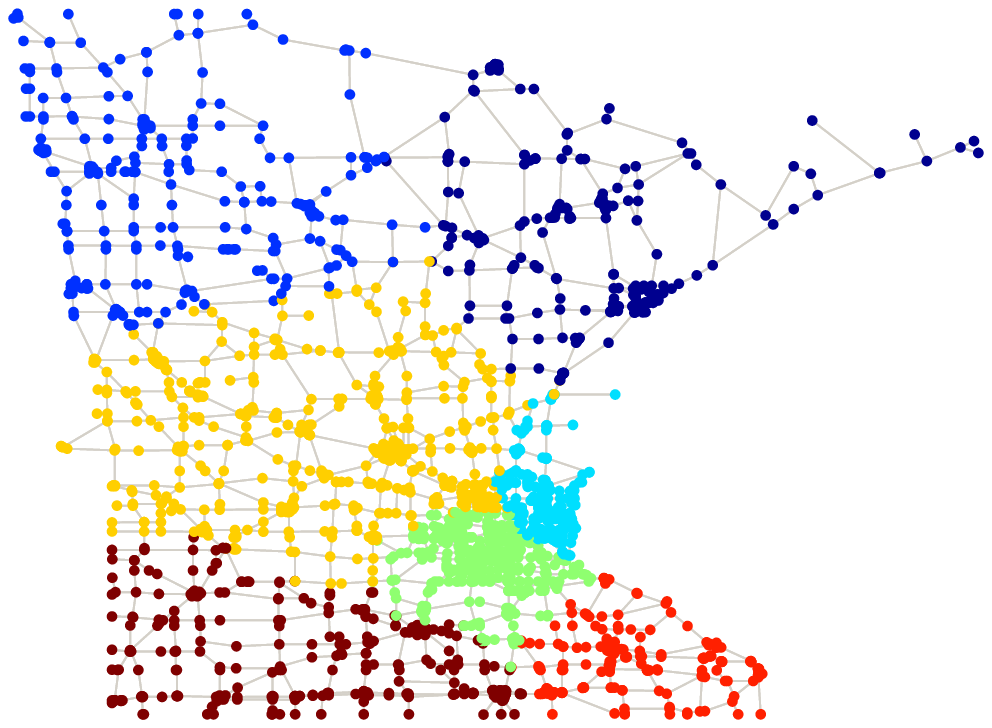}
		\vspace{-6mm}		
		\caption{$K=7$}
		\label{Fig_Minnesota_K7}
	\end{subfigure}%
	\centering
	\begin{subfigure}[b]{0.4\textwidth}
		\includegraphics[width=\textwidth]{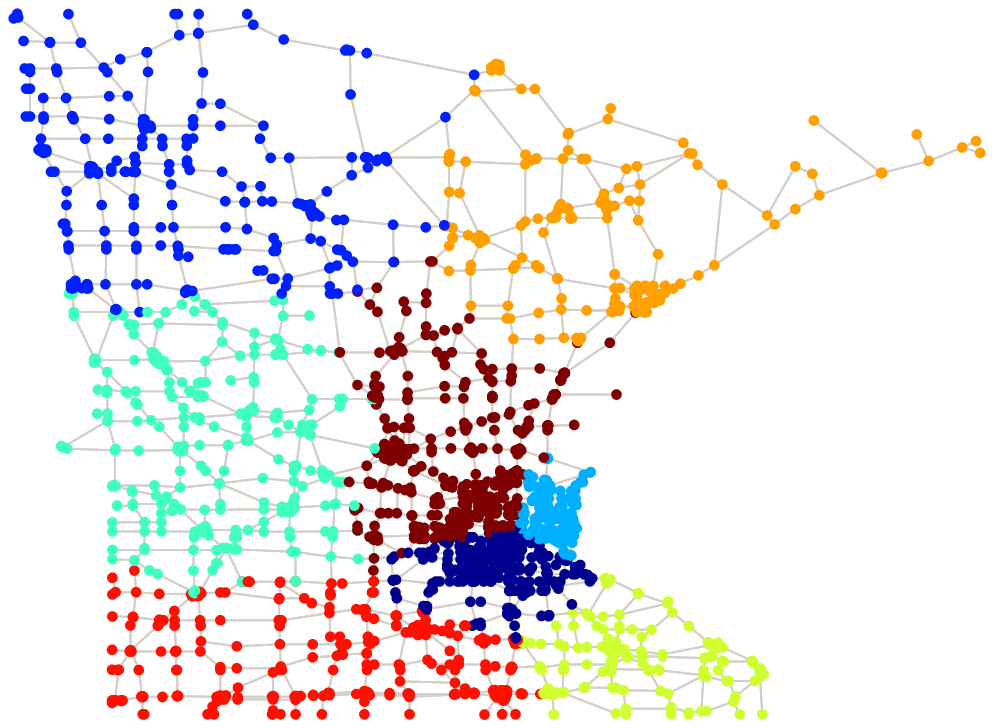}
		\vspace{-6mm}		
		\caption{$K=8$}
		\label{Fig_Minnesota_K8}
	\end{subfigure}%
	\\
	\centering
	\begin{subfigure}[b]{0.4\textwidth}
		\includegraphics[width=\textwidth]{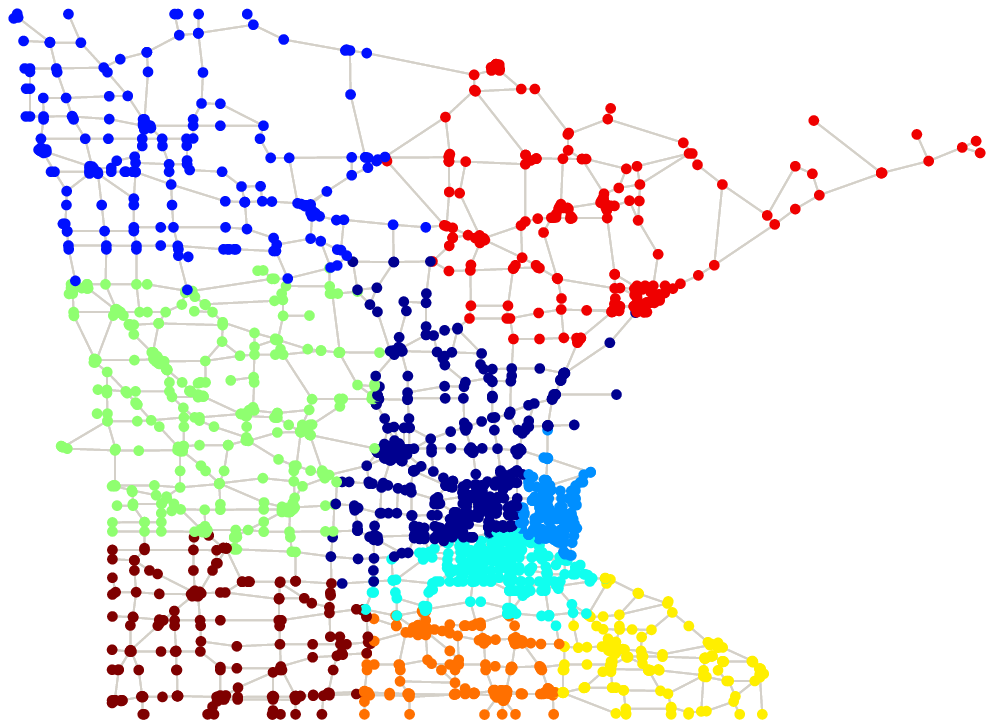}
		\vspace{-6mm}		
		\caption{$K=9$}
		\label{Fig_Minnesota_K9}
	\end{subfigure}%
	\centering
	\begin{subfigure}[b]{0.4\textwidth}	
		\includegraphics[width=\textwidth]{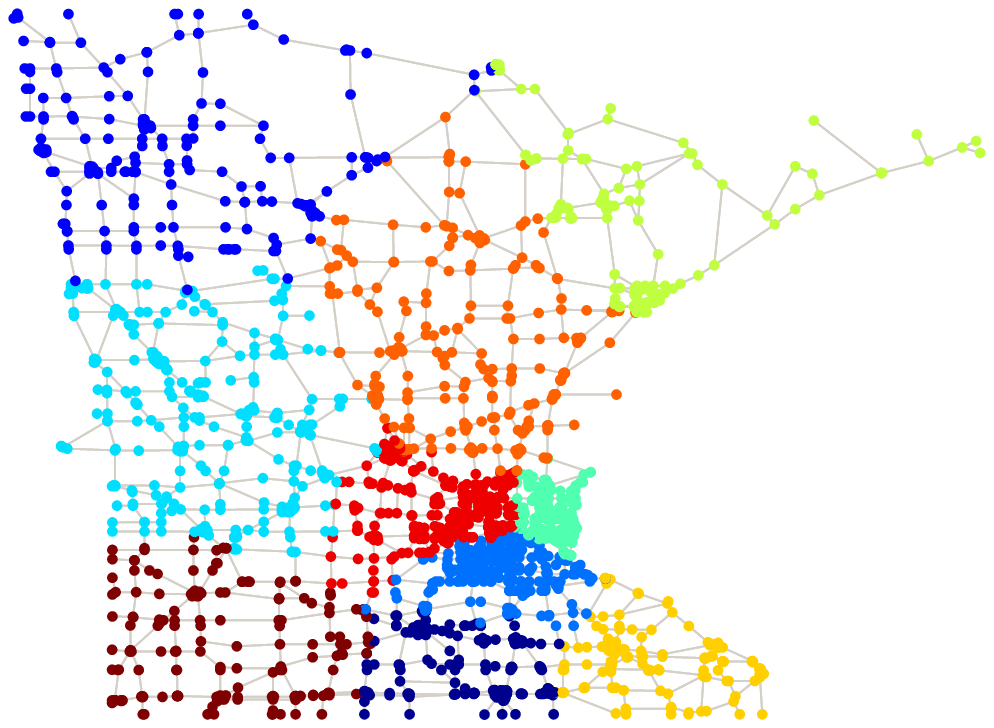}		
		\vspace{-6mm}	
		\caption{$K=10$}
		\label{Fig_Minnesota_K10}
	\end{subfigure}%
	\caption{Visualization of user-guided spectral clustering on Minnesota Road with respect to selected cluster count $K$. Colors represent different clusters.
	}
	\label{Fig_Minnesota_cluster_visualization}
\end{figure*}

We also apply the proposed incremental user-guided spectral
clustering  algorithm (\textbf{Algorithm}~\ref{algo_incremental_automated_clustering}) to
Power Grid, CLUTO, Swiss Roll, Youtube, and BlogCatalog. In
Fig. \ref{Fig_clustering_metric}, we show how the values of clustering metrics
change with $K$ for each dataset.  The incremental method enables us to efficiently generate
all clustering results with $K=2, 3, 4 \ldots$ and so on. It can be observed from Fig. \ref{Fig_clustering_metric} that for each dataset  the  clustering metric that exhibits the highest variation in $K$ can be different. This suggests that selecting the correct number of clusters is
	a difficult task and a user might need to use different clustering metrics for
	a range of $K$ values, and Incremental-IO is an
	effective tool to support such an endeavor.

\section{Conclusion}

In this paper we present Incremental-IO, an efficient incremental eigenpair
computation method for graph Laplacian matrices which works by transforming a
batch eigenvalue decomposition problem into a sequential leading eigenpair
computation problem. The method is elegant, robust and easy to implement using
a scientific programming language, such as Matlab.  We provide analytical proof
of its correctness. We also demonstrate that it achieves significant reduction
in computation time when compared with the batch computation method.
Particularly, it is observed that the difference in computation time of these
two methods grows polynomially as the graph size increases.

To demonstrate the effectiveness of Incremental-IO,
we also show experimental evidences that
obtaining such an incremental method by adapting the existing leading eigenpair
solvers (such as, the Lanczos algorithm) is non-trivial and such efforts generally do
not lead to a robust solution.

Finally, we demonstrate that the proposed incremental eigenpair computation
method (Incremental-IO) is an effective tool for a user-guided spectral clustering task, which
effectively updates clustering results and metrics for each increment of the
cluster count.

\section*{Acknowledgments}
{This research is sponsored by Mohammad Al Hasan's NSF CAREER Award (IIS-1149851). The contents are solely the responsibility of the authors and do not necessarily represent the official view of NSF. 
}
%
{ 
\bibliographystyle{abbrv}
\bibliography{SNAM_rev}

\begin{thebibliography}{10}

\bibitem{basu2004active}
S.~Basu, A.~Banerjee, and R.~J. Mooney.
\newblock Active semi-supervision for pairwise constrained clustering.
\newblock In {\em SDM}, volume~4, pages 333--344, 2004.

\bibitem{belkin2003laplacian}
M.~Belkin and P.~Niyogi.
\newblock Laplacian eigenmaps for dimensionality reduction and data
  representation.
\newblock {\em Neural computation}, 15(6):1373--1396, 2003.

\bibitem{blondel2008fast}
V.~D. Blondel, J.-L. Guillaume, R.~Lambiotte, and E.~Lefebvre.
\newblock Fast unfolding of communities in large networks.
\newblock {\em Journal of Statistical Mechanics: Theory and Experiment}, (10),
  2008.

\bibitem{calvetti1994implicitly}
D.~Calvetti, L.~Reichel, and D.~C. Sorensen.
\newblock An implicitly restarted lanczos method for large symmetric eigenvalue
  problems.
\newblock {\em Electronic Transactions on Numerical Analysis}, 2(1):21, 1994.

\bibitem{chen2016efficient}
J.~Chen, L.~Wu, K.~Audhkhasi, B.~Kingsbury, and B.~Ramabhadrari.
\newblock Efficient one-vs-one kernel ridge regression for speech recognition.
\newblock In {\em IEEE International Conference on Acoustics, Speech and Signal
  Processing (ICASSP)}, pages 2454--2458, 2016.

\bibitem{CPY14deep}
P.-Y. Chen and A.~Hero.
\newblock Deep community detection.
\newblock 63(21):5706--5719, Nov. 2015.

\bibitem{CPY14spectral}
P.-Y. Chen and A.~Hero.
\newblock Phase transitions in spectral community detection.
\newblock 63(16):4339--4347, Aug 2015.

\bibitem{CPY13GlobalSIP}
P.-Y. Chen and A.~O. Hero.
\newblock Node removal vulnerability of the largest component of a network.
\newblock In {\em GlobalSIP}, pages 587--590, 2013.

\bibitem{CPY14ComMag}
P.-Y. Chen and A.~O. Hero.
\newblock Assessing and safeguarding network resilience to nodal attacks.
\newblock 52(11):138--143, Nov. 2014.

\bibitem{chen2016phase}
P.-Y. Chen and A.~O. Hero.
\newblock Phase transitions and a model order selection criterion for spectral
  graph clustering.
\newblock {\em arXiv preprint arXiv:1604.03159}, 2016.

\bibitem{chen_tsipn}
P.-Y. Chen and A.~O. Hero.
\newblock Multilayer spectral graph clustering via convex layer aggregation:
  Theory and algorithms.
\newblock {\em IEEE Transactions on Signal and Information Processing over
  Networks}, 3(3):553--567, Sept 2017.

\bibitem{chen2017bias}
P.-Y. Chen and S.~Liu.
\newblock Bias-variance tradeoff of graph laplacian regularizer.
\newblock {\em IEEE Signal Processing Letters}, 24(8):1118--1122, Aug 2017.

\bibitem{chen2017revisiting}
P.-Y. Chen and L.~Wu.
\newblock Revisiting spectral graph clustering with generative community
  models.
\newblock {\em arXiv preprint arXiv:1709.04594}, 2017.

\bibitem{chen2016incremental}
P.-Y. Chen, B.~Zhang, M.~A. Hasan, and A.~O. Hero.
\newblock Incremental method for spectral clustering of increasing orders.
\newblock In {\em KDD Workshop on Mining and Learning with Graphs}, 2016.

\bibitem{icde-sutanay}
S.~Choudhury, K.~Agarwal, S.~Purohit, B.~Zhang, M.~Pirrung, W.~Smith, and
  M.~Thomas.
\newblock \uppercase{NOUS:} construction and querying of dynamic knowledge
  graphs.
\newblock In {\em Proceedings of 33rd IEEE International Conference on Data
  Engineering}, pages 1563--1565, 2017.

\bibitem{Chung97SpectralGraph}
F.~R.~K. Chung.
\newblock {\em Spectral Graph Theory}.
\newblock American Mathematical Society, 1997.

\bibitem{dhanjal2014efficient}
C.~Dhanjal, R.~Gaudel, and S.~Cl{\'e}men{\c{c}}on.
\newblock Efficient eigen-updating for spectral graph clustering.
\newblock {\em Neurocomputing}, 131:440--452, 2014.

\bibitem{icmla-dundar}
M.~Dundar, Q.~Kou, B.~Zhang, Y.~He, and B.~Rajwa.
\newblock Simplicity of kmeans versus deepness of deep learning: {A} case of
  unsupervised feature learning with limited data.
\newblock In {\em Proceedings of 14th IEEE International Conference on Machine
  Learning and Applications}, pages 883--888, 2015.

\bibitem{Hasan06linkprediction}
M.~A. Hasan, V.~Chaoji, S.~Salem, and M.~Zaki.
\newblock Link prediction using supervised learning.
\newblock In {\em In Proc. of SDM 06 workshop on Link Analysis,
  Counterterrorism and Security}, 2006.

\bibitem{Hasan2011}
M.~A. Hasan and M.~J. Zaki.
\newblock {\em A Survey of Link Prediction in Social Networks}, pages 243--275.
\newblock Springer US, 2011.

\bibitem{HornMatrixAnalysis}
R.~A. Horn and C.~R. Johnson.
\newblock {\em {Matrix Analysis}}.
\newblock Cambridge University Press, 1990.

\bibitem{jia2009incremental}
P.~Jia, J.~Yin, X.~Huang, and D.~Hu.
\newblock Incremental laplacian eigenmaps by preserving adjacent information
  between data points.
\newblock {\em Pattern Recognition Letters}, 30(16):1457--1463, 2009.

\bibitem{Krzakala2013}
F.~Krzakala, C.~Moore, E.~Mossel, J.~Neeman, A.~Sly, L.~Zdeborova, and
  P.~Zhang.
\newblock Spectral redemption in clustering sparse networks.
\newblock {\em Proc. National Academy of Sciences}, 110:20935--20940, 2013.

\bibitem{kuczynski1992estimating}
J.~Kuczynski and H.~Wozniakowski.
\newblock Estimating the largest eigenvalue by the power and lanczos algorithms
  with a random start.
\newblock {\em SIAM journal on matrix analysis and applications},
  13(4):1094--1122, 1992.

\bibitem{lanczos1950iteration}
C.~Lanczos.
\newblock An iteration method for the solution of the eigenvalue problem of
  linear differential and integral operators.
\newblock {\em Journal of Research of the National Bureau of Standards}, 45(4),
  1950.

\bibitem{larsen2000computing}
R.~M. Larsen.
\newblock Computing the svd for large and sparse matrices.
\newblock {\em SCCM, Stanford University, June}, 16, 2000.

\bibitem{lehoucq1998arpack}
R.~B. Lehoucq, D.~C. Sorensen, and C.~Yang.
\newblock {\em ARPACK users' guide: solution of large-scale eigenvalue problems
  with implicitly restarted Arnoldi methods}, volume~6.
\newblock Siam, 1998.

\bibitem{liu2017genome}
S.~Liu, H.~Chen, S.~Ronquist, L.~Seaman, N.~Ceglia, W.~Meixner, L.~A. Muir,
  P.-Y. Chen, G.~Higgins, P.~Baldi, et~al.
\newblock Genome architecture leads a bifurcation in cell identity.
\newblock {\em bioRxiv}, page 151555, 2017.

\bibitem{liu2017accelerated}
S.~Liu, P.-Y. Chen, and A.~O. Hero.
\newblock Accelerated distributed dual averaging over evolving networks of
  growing connectivity.
\newblock {\em arXiv preprint arXiv:1704.05193}, 2017.

\bibitem{weiyi}
W.~Liu, P.-Y. Chen, S.~Yeung, T.~Suzumura, and L.~Chen.
\newblock Principled multilayer network embedding.
\newblock {\em CoRR}, abs/1709.03551, 2017.

\bibitem{zifan1}
W.~J. Lu, C.~Xu, Z.~Pei, A.~S. Mayhoub, M.~Cushman, and D.~A. Flockhart.
\newblock The tamoxifen metabolite norendoxifen is a potent and selective
  inhibitor of aromatase (\uppercase{CYP19}) and a potential lead compound for
  novel therapeutic agents.
\newblock {\em Breast Cancer Research and Treatment}, 133(1):99--109, 2012.

\bibitem{Luxburg07}
U.~Luxburg.
\newblock A tutorial on spectral clustering.
\newblock {\em Statistics and Computing}, 17(4):395--416, Dec. 2007.

\bibitem{Merris94}
R.~Merris.
\newblock Laplacian matrices of graphs: a survey.
\newblock {\em Linear Algebra and its Applications}, 197-198:143--176, 1994.

\bibitem{ng2002spectral}
A.~Y. Ng, M.~I. Jordan, and Y.~Weiss.
\newblock On spectral clustering: Analysis and an algorithm.
\newblock In {\em NIPS}, pages 849--856, 2002.

\bibitem{ning2007incremental}
H.~Ning, W.~Xu, Y.~Chi, Y.~Gong, and T.~S. Huang.
\newblock Incremental spectral clustering with application to monitoring of
  evolving blog communities.
\newblock In {\em SDM}, pages 261--272, 2007.

\bibitem{ning2010incremental}
H.~Ning, W.~Xu, Y.~Chi, Y.~Gong, and T.~S. Huang.
\newblock Incremental spectral clustering by efficiently updating the
  eigen-system.
\newblock {\em Pattern Recognition}, 43(1):113--127, 2010.

\bibitem{Olfati-Saber07}
R.~Olfati-Saber, J.~Fax, and R.~Murray.
\newblock Consensus and cooperation in networked multi-agent systems.
\newblock 95(1):215--233, 2007.

\bibitem{parlett1980symmetric}
B.~N. Parlett.
\newblock {\em The symmetric eigenvalue problem}, volume~7.
\newblock SIAM, 1980.

\bibitem{zifan2}
Z.~Pei, Y.~Xiao, J.~Meng, A.~Hudmon, and T.~R. Cummins.
\newblock Cardiac sodium channel palmitoylation regulates channel availability
  and myocyte excitability with implications for arrhythmia generation.
\newblock {\em Nature Communications}, 7, 2016.

\bibitem{Polito01grouping}
M.~Polito and P.~Perona.
\newblock Grouping and dimensionality reduction by locally linear embedding.
\newblock In {\em NIPS}, 2001.

\bibitem{Poon2012}
L.~K.~M. Poon, A.~H. Liu, T.~Liu, and N.~L. Zhang.
\newblock A model-based approach to rounding in spectral clustering.
\newblock In {\em UAI}, pages 68--694, 2012.

\bibitem{pothen1990partitioning}
A.~Pothen, H.~D. Simon, and K.-P. Liou.
\newblock Partitioning sparse matrices with eigenvectors of graphs.
\newblock {\em SIAM journal on matrix analysis and applications},
  11(3):430--452, 1990.

\bibitem{Radicchi13}
F.~Radicchi and A.~Arenas.
\newblock Abrupt transition in the structural formation of interconnected
  networks.
\newblock {\em Nature Physics}, 9(11):717--720, Nov. 2013.

\bibitem{ranjan2014incremental}
G.~Ranjan, Z.-L. Zhang, and D.~Boley.
\newblock Incremental computation of pseudo-inverse of laplacian.
\newblock In {\em Combinatorial Optimization and Applications}, pages 729--749.
  Springer, 2014.

\bibitem{Saade2015spectral}
A.~Saade, F.~Krzakala, M.~Lelarge, and L.~Zdeborova.
\newblock Spectral detection in the censored block model.
\newblock {\em arXiv:1502.00163}, 2015.

\bibitem{Zhang.Hasan.ea:15}
T.~K. Saha, B.~Zhang, and M.~Al~Hasan.
\newblock Name disambiguation from link data in a collaboration graph using
  temporal and topological features.
\newblock {\em Social Network Analysis Mining}, 5(1):11:1--11:14, 2015.

\bibitem{Shi00}
J.~Shi and J.~Malik.
\newblock Normalized cuts and image segmentation.
\newblock 22(8):888--905, 2000.

\bibitem{Shuman13}
D.~Shuman, S.~Narang, P.~Frossard, A.~Ortega, and P.~Vandergheynst.
\newblock The emerging field of signal processing on graphs: Extending
  high-dimensional data analysis to networks and other irregular domains.
\newblock 30(3):83--98, 2013.

\bibitem{van2001graph}
S.~M. Van~Dongen.
\newblock {\em Graph clustering by flow simulation}.
\newblock PhD thesis, University of Utrecht, 2000.

\bibitem{White05}
S.~White and P.~Smyth.
\newblock A spectral clustering approach to finding communities in graph.
\newblock In {\em SDM}, volume~5, pages 76--84, 2005.

\bibitem{wu2000thick}
K.~Wu and H.~Simon.
\newblock Thick-restart lanczos method for large symmetric eigenvalue problems.
\newblock {\em SIAM Journal on Matrix Analysis and Applications},
  22(2):602--616, 2000.

\bibitem{wu2016estimating}
L.~Wu, J.~Laeuchli, V.~Kalantzis, A.~Stathopoulos, and E.~Gallopoulos.
\newblock Estimating the trace of the matrix inverse by interpolating from the
  diagonal of an approximate inverse.
\newblock {\em Journal of Computational Physics}, 326:828--844, 2016.

\bibitem{wu2009empirical}
L.~Wu, M.~Q.-H. Meng, Z.~Dong, and H.~Liang.
\newblock An empirical study of dv-hop localization algorithm in random sensor
  networks.
\newblock In {\em International Conference on Intelligent Computation
  Technology and Automation}, volume~4, pages 41--44, 2009.

\bibitem{wu2017primme_svds}
L.~Wu, E.~Romero, and A.~Stathopoulos.
\newblock Primme\_svds: A high-performance preconditioned svd solver for
  accurate large-scale computations.
\newblock {\em SIAM Journal on Scientific Computing}, 39(5):S248--S271, 2017.

\bibitem{wu2015preconditioned}
L.~Wu and A.~Stathopoulos.
\newblock A preconditioned hybrid svd method for accurately computing singular
  triplets of large matrices.
\newblock {\em SIAM Journal on Scientific Computing}, 37(5):S365--S388, 2015.

\bibitem{wu2016revisiting}
L.~Wu, I.~E. Yen, J.~Chen, and R.~Yan.
\newblock Revisiting random binning features: Fast convergence and strong
  parallelizability.
\newblock In {\em ACM SIGKDD International Conference on Knowledge Discovery
  and Data Mining}, pages 1265--1274, 2016.

\bibitem{Zaki.Jr:14}
M.~J. Zaki and W.~M. Jr.
\newblock {\em Data Mining and Analysis: Fundamental Concepts and Algorithms}.
\newblock Cambridge University Press, 2014.

\bibitem{zelnik2004self}
L.~Zelnik-Manor and P.~Perona.
\newblock Self-tuning spectral clustering.
\newblock In {\em NIPS}, pages 1601--1608, 2004.

\bibitem{Zhang.Hasan.ea:16}
B.~Zhang, S.~Choudhury, M.~A. Hasan, X.~Ning, K.~Agarwal, and S.~P. andy Paola
  Pesantez~Cabrera.
\newblock Trust from the past: Bayesian personalized ranking based link
  prediction in knowledge graphs.
\newblock In {\em SDM MNG Workshop}, 2016.

\bibitem{Zhang.Dundar.ea:16}
B.~Zhang, M.~Dundar, and M.~A. Hasan.
\newblock Bayesian non-exhaustive classification a case study: Online name
  disambiguation using temporal record streams.
\newblock In {\em Proceedings of the 25th ACM International on Conference on
  Information and Knowledge Management}, pages 1341--1350. ACM, 2016.

\bibitem{tkde_online}
B.~Zhang, M.~Dundar, and M.~A. Hasan.
\newblock Bayesian non-exhaustive classification for active online name
  disambiguation.
\newblock {\em arXiv preprint arXiv:1702.02287}, 2017.

\bibitem{Zhang.Hasan.3}
B.~Zhang and M.~A. Hasan.
\newblock Name disambiguation in anonymized graphs using network embedding.
\newblock In {\em Proceedings of the 26th ACM International on Conference on
  Information and Knowledge Management}, 2017.

\bibitem{Zhang.Noman.ea:17}
B.~Zhang, N.~Mohammed, V.~S. Dave, and M.~A. Hasan.
\newblock Feature selection for classification under anonymity constraint.
\newblock {\em Transactions on Data Privacy}, 10(1):1--25, 2017.

\bibitem{ZhangSaha14}
B.~Zhang, T.~K. Saha, and M.~A. Hasan.
\newblock Name disambiguation from link data in a collaboration graph.
\newblock In {\em ASONAM}, pages 81--84, 2014.

\end{thebibliography}
}

\clearpage

\end{document}